\documentclass{article} %
\usepackage{iclr2027_conference,times}

\iclrfinalcopy

\usepackage{amsmath,amsfonts,bm}

\def\eqref#1{equation~\ref{#1}}

\def\1{\bm{1}}

\DeclareMathAlphabet{\mathsfit}{\encodingdefault}{\sfdefault}{m}{sl}
\SetMathAlphabet{\mathsfit}{bold}{\encodingdefault}{\sfdefault}{bx}{n}

\DeclareMathOperator*{\argmax}{arg\,max}
\DeclareMathOperator*{\argmin}{arg\,min}

\usepackage{hyperref}
\usepackage{url}

\usepackage[utf8]{inputenc} %
\usepackage[T1]{fontenc}    %
\usepackage{hyperref}       %
\usepackage{url}            %
\usepackage{booktabs}       %
\usepackage{amsfonts}       %
\usepackage{nicefrac}       %
\usepackage{microtype}      %
\usepackage{xcolor}         %

\usepackage{amsmath,amsthm,bbm,bm}
\usepackage{multirow,makecell}
\usepackage{aliascnt,enumitem}
\usepackage{algorithm,algpseudocode}
\usepackage{graphicx,subcaption}
\usepackage{hyperref,cleveref}

\newtheorem{definition}{Definition}
\newtheorem{assumption}{Assumption}
\newtheorem{lemma}{Lemma}
\newtheorem*{remark}{Remark}
\newtheorem{theorem}{Theorem}

\newaliascnt{proposition}{theorem}
\newtheorem{proposition}[proposition]{Proposition}
\aliascntresetthe{proposition}

\newaliascnt{corollary}{theorem}
\newtheorem{corollary}[corollary]{Corollary}
\aliascntresetthe{corollary}

\crefname{theorem}{Theorem}{Theorems}
\Crefname{theorem}{Theorem}{Theorems}
\crefname{proposition}{Proposition}{Propositions}
\Crefname{proposition}{Proposition}{Propositions}
\crefname{corollary}{Corollary}{Corollaries}
\Crefname{corollary}{Corollary}{Corollaries}
\crefname{definition}{Definition}{Definitions}
\Crefname{definition}{Definition}{Definitions}
\crefname{assumption}{Assumption}{Assumptions}
\Crefname{assumption}{Assumption}{Assumptions}
\crefname{lemma}{Lemma}{Lemmas}
\Crefname{lemma}{Lemma}{Lemmas}

\crefname{equation}{Eq.}{Eqs.}
\crefname{table}{Tab.}{Tabs.}
\crefname{figure}{Fig.}{Figs.}
\crefname{algorithm}{Algorithm}{Algorithms}
\crefname{section}{Section}{Sections}
\Crefname{section}{Section}{Sections}

\newcommand{\st}{\textnormal{s.t.}\ }
\newcommand{\eg}{\textit{e.g.},\ }
\newcommand{\vs}{\textit{v.s.},\ }
\newcommand{\ie}{\textit{i.e.},\ }
\newcommand{\iid}{\textit{i.i.d.}\ }
\newcommand{\wrt}{\textnormal{w.r.t.}\ }

\title{Dynamic Minimax Regret Optimization\\for Robust LLM Post-Training}

\author{Chengbo Zang$^1$, Haoyu Dong$^{1,2}$, Mehmet Kerem Turkcan$^3$,\\
\textbf{Gil Zussman$^1$, Zoran Kostic$^1$, Javad Ghaderi$^1$}\\
$^1$Department of Electrical Engineering, Columbia University\\
$^2$Department of Computer Science \& Engineering, University of California, Santa Cruz\\
$^3$Department of Civil Engineering, Columbia University\\
\texttt{\{cz2678,hd2573,mkt2126,gz2136,zk2172,jg3465\}@columbia.edu}
}

\begin{document}

\maketitle
\lhead{Preprint}

\begin{abstract}
Modern LLM training increasingly relies on heterogeneous data sources spanning different domains, tasks, preference distributions, and difficulty levels.
We study \textit{dynamic minimax regret} for group-distributionally robust LLM post-training under instantaneous mini-batch-only bandit feedback.
The framework views the training as a two-player sampler-optimizer process: a sampler adaptively selects among data sources using bandit feedback, while an optimizer updates the model parameters using stochastic gradients from the selected source.
We focus on the practically restrictive setting where source losses evolve with model training but historical data are not re-evaluated, requiring the sampler to track instantaneous worst-sources from stale partial feedback.
We propose DUCB-OGD, a simple and scalable algorithm that couples a Discounted Upper-Confidence-Bound sampler with an Online Gradient Descent optimizer.
The sampler maintains exponential moving average loss estimates and confidence radii based on discounted effective sample sizes, avoiding costly re-evaluation of past data or intrusive changes to standard training pipelines.
For $K$ data sources and $T$ training steps, we prove that DUCB-OGD achieves a dynamic minimax regret of $\tilde{O}(K^{1/4}T^{3/4})$, which is optimal up to logarithmic factors for the undiscounted objective under our feedback model.
Extensive experiments across supervised fine-tuning, preference optimization, and reinforcement learning show that DUCB-OGD integrates seamlessly into modern LLM training pipelines and improves worst-group robustness with negligible computational overhead compared with standard sampling baselines.
\end{abstract}

\section{Introduction}
\label{sec:intro}

Diverse data mixtures are central to modern Large Language Model (LLM) post-training. For example, supervised fine-tuning data may combine instruction-following, coding, mathematical reasoning, safety, and domain-specific examples; preference datasets may reflect different annotator populations or task families; and reinforcement learning objectives may expose the model to prompts whose rewards evolve unevenly during training~\citep{ouyang_training_2022,rafailov_direct_2023,schulman_proximal_2017,shao_deepseekmath_2024}.
As a result, post-training is not merely an average-loss minimization problem. A model can improve substantially on common or easy sources, while remaining weak on smaller, harder, or temporarily under-trained groups.
This makes robust performance across data sources a central objective in LLM post-training. In practical deployment, worst-source robustness is a complementary and equally important criterion to the commonly reported mean performance, but often under-emphasized in aggregate evaluations.

Group-distributionally robust optimization (GDRO) provides a principled formulation for this goal by minimizing the worst-case expected loss over a predefined set of data sources~\citep{duchi_learning_2021,sagawa_distributionally_2020,oren_distributionally_2019}.
Informally, GDRO seeks model parameters that perform well on the most challenging source rather than the average performance across sources.
However, the standard GDRO viewpoint is essentially static: robustness is typically evaluated with respect to the worst-performing source at the final or time-averaged solution, or determined in hindsight over the training trajectory~\citep{zhang_stochastic_2023,soma_near-optimal_2025}.
This static benchmark does not fully capture the shift of the worst-performing source as the model evolves in LLM post-training.
A robust algorithm should therefore track and respond to the evolving worst source, rather than optimize against a fixed source chosen in hindsight.

Recent work has shown that mixture proportions can strongly affect pretraining and post-training performance, motivating both offline mixture optimization and online adaptive sampling methods~\citep{xie_doremi_2023,chen_skill-it_2023,ye_data_2025,corrado_automixalign_2025,shin_dynamixsft_2025}.
Such mechanisms can be expensive or intrusive in large-scale post-training systems, where throughput, memory efficiency, and compatibility with distributed training libraries are critical.
This creates a gap between the theoretical appeal of GDRO and the practical constraints of LLM post-training: we need adaptive algorithms that are robust, theoretically grounded, and lightweight enough to integrate into standard training pipelines relying on well-optimized deep learning libraries. We discuss related work in detail in Appendix~\ref{sec:review}.

In this paper, we study \textit{dynamic minimax regret} for group-distributionally robust LLM post-training under instantaneous mini-batch-only bandit feedback.
The training trajectory can be naturally decomposed into a \textit{sampler}, which selects a source using only its observed mini-batch loss, and an \textit{optimizer}, which updates the model using the corresponding stochastic gradient.
We propose DUCB-OGD, a simple algorithm that combines a Discounted Upper-Confidence-Bound (DUCB)~\citep{kivinen_upper-confidence_2011} sampler with an Online Gradient Descent (OGD)~\citep{hazan_introduction_2016} optimizer.
In contrast, while~\citet{zang_adaptive_2025} already study dynamic worst-source tracking, they refresh source-loss estimates by re-evaluating accumulated historical samples.
Our replay-free formulation instead uses fresh mini-batches as unbiased stochastic estimates of population loss, adds no auxiliary model evaluations, and requires only $O(K)$ scalar state, at the cost of tracking unobserved sources from stale feedback.
We show that this simple sampler-optimizer design achieves dynamic minimax regret of $\tilde{O}(K^{1/4}T^{3/4})$ for $K$ data sources over $T$ training steps, which is optimal for the undiscounted objective up to logarithmic factors under our feedback model.

Our main contributions are summarized as follows:
\begin{enumerate}[leftmargin=2em,topsep=0em,itemsep=0em]
    \item We formulate robust LLM post-training under \textit{instantaneous mini-batch-only feedback}, where evolving source losses must be tracked from stale observations without historical replay or auxiliary model evaluations.
    \item We propose DUCB-OGD, an efficient algorithm that combines DUCB sampling with OGD optimization, and prove that it achieves a dynamic minimax regret of $\tilde{O}(K^{1/4}T^{3/4})$, matching the optimal rate up to $O(\log T)$ factors for the undiscounted objective under our feedback model.
    \item We evaluate the method across supervised fine-tuning, preference optimization, and reinforcement learning, showing improved worst-source robustness under negligible computational overhead against standard and heuristic sampling baselines.
\end{enumerate}

\section{Problem Setup}
\label{sec:setup}

\begin{table}[t]
    \caption{GDRO formulations of LLM post-training objectives and estimated worst-source (WS) switching. The final column reports the mean number of switches per 1,000 steps, with the range across experiments in brackets. The worst source has the largest exponentially smoothed loss for SFT/DPO or the lowest exponentially smoothed reward for GRPO. Additional optimization dynamics are reported in \cref{sec:expr,sec:results-opt}.}
    \label{tab:prelim-objective}
    \centering
    \begin{tabular}{lcc}
        \toprule
        \textbf{Objective} & $\bm{-\ell(\theta,\cdot)}$ & \shortstack{\textbf{WS Switches / 1k Steps}} \\
        \midrule
        SFT &
            $\mathbb{E}_{o\sim\mathcal{D}^k(\cdot\mid q)}\left[\log\pi_\theta(o\mid q)\right]$ &
            \textbf{41}~[3,124] \\
        DPO &
            $\mathbb{E}_{(o^+,o^-)\sim\mathcal{D}^k(\cdot\mid q)}\left[\log p_\theta^{\mathrm{BT}}(o^+\succ o^-\mid q)\right]$ &
            \textbf{31}~[13,53] \\
        RL (GRPO) &
            $\mathbb{E}_{o\sim\pi_\theta(\cdot\mid q)}\left[r(q,o)\right]$ &
            \textbf{20}~[0,32] \\
        \bottomrule
    \end{tabular}
\end{table}

Let $\mathbb{X}$ be a data space with $K$ \textit{data sources} indexed by $[K]:=\{1,\dots,K\}$. Each source $k\in[K]$ is associated with a distribution $\mathcal{D}^k$ supported on $\mathbb{X}$, and each query to source $k$ returns \iid samples from $\mathcal{D}^k$.
Let $\theta$ be the learnable parameter from a parameter space $\mathbf{\Theta}$. Let $\ell:\mathbf{\Theta}\times\mathbb{X}\to\mathbb{R}$ be a differentiable loss function.
The \textit{expected loss} of source $k$ at the model parameter $\theta$ is defined as $\mu^k(\theta):=\mathbb{E}_{X\sim\mathcal{D}^k} [\ell(\theta,X)]$. 
A global minimax solution is any optimizer of the GDRO objective, i.e.,
\begin{equation}\label{eq:gdro}
    \theta^\star \in \argmin_{\theta\in\mathbf{\Theta}} \max_{k\in[K]} \mu^k(\theta).
\end{equation}

This formulation captures a broad class of LLM post-training objectives through appropriate choices of the source distributions and loss functions as described in~\cref{tab:prelim-objective}.
Let $\pi_\theta$ be an LLM, and $q,o$ be a pair of prompt and response. Then $\pi_\theta(o|q)$ is the probability that the LLM generates the response $o$ given the prompt $q$.
In supervised fine-tuning (SFT), each source generates prompt-response pairs, and the loss is the negative log-likelihood of the target response~\citep{ouyang_training_2022}.
For direct preference optimization (DPO), let $(o^+,o^-)$ denote a preferred and dispreferred response pair. Let $p_\theta^{\mathrm{BT}}(o^+\succ o^-|q)$ denote the Bradley-Terry probability~\citep{rafailov_direct_2023,bradley_rank_1952} that the model prefers $o^+$ to $o^-$, and the loss is the negative Bradley-Terry preference probability.
For reinforcement learning (RL), let $r(q,o)$ denote the reward assigned to the response $o$ for the prompt $q$. Each source generates prompts, the model induces the response distribution, and the loss is the negative reward~\citep{schulman_proximal_2017,shao_deepseekmath_2024}.
For simplicity, \cref{tab:prelim-objective} presents these objectives at the sequence level, omitting token-wise factorization, reference-model regularization, and implementation-specific normalization terms.

We consider an iterative training algorithm $\mathcal{A}$, consisting of a \textit{sampler} and an \textit{optimizer}, that runs for $T$ steps. At each step $t\in[T]$, the sampler specifies  a distribution $\mathcal{P}_t$ over the data-source indices $[K]$. A source $k_t\sim\mathcal{P}_t$ is then selected, and a mini-batch $\mathcal{B}_t\subseteq\mathbb{X}^{k_t}$ of \iid samples is drawn from the corresponding distribution $\mathcal{D}^{k_t}$.
The model observes $\mathcal{B}_t$ and calculates the empirical loss 
$\hat{\mu}^{k_t}(\theta_t;\mathcal{B}_t):=\sum_{X\in\mathcal{B}_t}\ell(\theta_t,X)/|\mathcal{B}_t|$ based on the current model parameter $\theta_t$. Then the optimizer updates the parameter to obtain $\theta_{t+1}$ for the next step.
The sampler observes only $\hat{\mu}^{k_t}(\theta_t;\mathcal{B}_t)$; it neither observes the current losses of unselected sources nor re-evaluates historical mini-batches.
After $T$ steps, the algorithm generates a trajectory of data source indices and model parameters~$\mathcal{A}_T:=\{(k_t,\theta_t),t=1,\dots,T\}$. Our goal is to evaluate this trajectory by how well it controls the \textit{adaptive} sequence of maximum source-wise losses throughout the training.
Following the notion of dynamic regret in online learning, the (undiscounted) \textit{dynamic minimax regret} is defined as follows.

\begin{definition}[Undiscounted Dynamic Minimax Regret]
    \label{def:undiscount-minimax-regret}
    Let $\theta^\star$ be the global minimax optimum in \cref{eq:gdro}.
    The undiscounted dynamic minimax regret of an algorithm $\mathcal{A}$ after $T$ steps is defined as
    \begin{equation}
        R(\mathcal{A}_T) := \sum_{t=1}^T \max_{k\in[K]} ~ \mu^k(\theta_t) - T \max_{k\in[K]} ~ \mu^k(\theta^\star).
    \end{equation}
\end{definition}

This definition compares the worst-case loss incurred by $\mathcal{A}_T$ with the best achievable minimax loss.
Crucially, the maximizing source in $\max_k\mu^k(\theta_t)$ may vary with $t$, which makes the regret \textit{dynamic}~\citep{zang_adaptive_2025}, \ie at each step, it reflects the source that is currently under-performing, rather than a single fixed source chosen in hindsight~\citep{zhang_stochastic_2023,soma_near-optimal_2025}.
To analyze the sampler-optimizer structure, we decompose the regret into an optimization component and a bandit component. 
The optimization component measures how well the optimizer competes with the minimax solution on the sampled sources. 
The bandit component measures how well the sampler tracks the instantaneous worst-performing source. 
We also introduce discount factors $\rho,\gamma\in[0,1]$ to emphasize recent observations, which is natural in post-training because source-wise losses evolve as the model changes.

\begin{definition}[Optimization Regret and Dynamic Bandit Regret]
    \label{def:optimization-bandit-regret}
    Let $\rho,\gamma\in[0,1]$.
    The discounted optimization regret of an algorithm $\mathcal{A}$ after $T$ steps is defined as
    \begin{equation}
        R_o^\rho(\mathcal{A}_T) := \sum_{t=1}^T \rho^{T-t} \Big( \mu^{k_t}(\theta_t) - \mu^{k_t}(\theta^\star) \Big),
    \end{equation}
    and the discounted dynamic bandit regret is defined as
    \begin{equation}
        R_b^\gamma(\mathcal{A}_T) := \sum_{t=1}^T \gamma^{T-t} \Big( \max_{k\in[K]} \mu^k(\theta_t) - \mu^{k_t}(\theta_t) \Big).
    \end{equation}
\end{definition}

The optimization regret $R_o^\rho$ captures the loss incurred by following the optimizer trajectory instead of using the minimax solution on the sampled sources.
The bandit regret $R_b^\gamma$ captures the loss incurred by sampling $k_t$ instead of the instantaneous worst-performing source at time $t$.
The following decomposition connects these quantities to the original undiscounted regret. Its proof simply follows from the fact that $\max_k\mu^k(\theta)\ge\mu^{k_t}(\theta)$ for all $\theta$ and $t$ and is provided in Appendix~\ref{sec:proofs-decom}.

\begin{proposition}[Regret Decomposition]
    Any trajectory $\mathcal{A}_T$ satisfies $R(\mathcal{A}_T)\le R_o^1(\mathcal{A}_T)+R_b^1(\mathcal{A}_T)$.
    \label{th:regret-decomposition}
\end{proposition}

Finally, we define the normalized discounted dynamic minimax regret as follows.

\begin{definition}[Dynamic Minimax Regret]
    For any $\rho,\gamma\in[0,1]$, the normalized $(\rho,\gamma)$-discounted dynamic minimax regret of an algorithm $\mathcal{A}$ after $T$ steps is defined as
    \begin{equation}
        R^{\rho,\gamma}(\mathcal{A}_T) := \frac{T}{N_T(\rho)} R_o^\rho(\mathcal{A}_T) + \frac{T}{N_T(\gamma)} R_b^\gamma(\mathcal{A}_T),
    \end{equation}
    where $N_T(\beta):=\sum_{t=1}^T\beta^{T-t}$ for any $\beta\in[0,1]$.
    \label{def:minimax-regret}
\end{definition}

The normalization converts each discounted sum to the scale of a length-$T$ horizon, making regret bounds comparable across choices of $\rho$ and $\gamma$. Here, $\rho$ and $\gamma$ separately discount the optimization and bandit components, whereas $\beta$ will be selected by the analyses.
When $\rho=\gamma=1$, the normalized regret reduces to the undiscounted sampler-optimizer decomposition, and by~\cref{th:regret-decomposition}, $R(\mathcal{A}_T)\le R^{1,1}(\mathcal{A}_T)$.
For $\rho,\gamma<1$, the discounted statistics used in our analysis can be viewed as exponential moving-average (EMA) estimates of past observations. This connects our theoretical quantities to EMA-based state tracking, which is widely used in deep learning optimization and training diagnostics~\cite{polyak_acceleration_1992,kingma_adam_2014}.

We make the following standard assumptions for the theoretical analysis. 
Specifically,~\cref{as:loss,as:domain} are required for analyzing gradient descent, and~\cref{as:iid} is required to ensure that the empirical loss and gradient estimators are unbiased.

\begin{assumption}[Finite Optimization Domain]
    There exists a closed convex $\mathbf{\Theta}^\prime$ and $D>0$ such that, for any $\theta,\theta^\prime$ generated by the algorithm, there is $\theta,\theta^\prime\in\mathbf{\Theta}^\prime$ and $\|\theta-\theta^\prime\|\le D$.
    \label{as:domain}
\end{assumption}

\begin{assumption}[Bounded Lipschitz Loss]
    There exists $C,G>0$ such that, for all $\theta\in\mathbf{\Theta}^\prime$ and $X\in\mathbb{X}$, $\ell(\theta,X)\in[0,C]$ and $\ell$ is $G$-Lipschitz in $\theta$.
    \label{as:loss}
\end{assumption}

\begin{assumption}[Source-Wise IID Sampling]
    All data (\ie prompts and responses, if applicable) collected by querying $\mathcal{D}^k$ are sampled \iid for all $k\in[K]$.
    \label{as:iid}
\end{assumption}

We interpret the Lipschitz assumption as a regularity condition along the realized training trajectory rather than as a claim over the entire parameter space. Standard implementation techniques such as gradient clipping, KL regularization, and bounded reward design help control the observed loss variation as is illustrated in Appendix~\ref{sec:results-opt}, providing empirical evidence consistent with the assumptions.
Furthermore, the low-epoch regime of LLM post-training ensures that most training samples are observed at most once, which is aligned with the source-wise \iid assumption.

\section{Method}
\label{sec:method}

\subsection{DUCB-OGD: High-level Description}
\label{sec:method-alg}

The training of modern deep neural networks, especially LLMs, relies heavily on well-optimized libraries. A standard pipeline typically involves a \textit{sampler}, which selects a mini-batch of data samples as the model inputs, and an \textit{optimizer}, which updates the model parameters using gradients computed from those mini-batches.
Thus, we consider a two-player sampler-optimizer framework for seamless deployment in real-world training systems.

\textbf{Optimizer.}
Recall that in each step $t$, the sampler generates a mini-batch of training samples $\mathcal{B}_t$ from $\mathcal{D}^{k_t}$. The optimizer updates the model parameters according to
\begin{equation}
    \theta_{t+1} \gets \theta_t - \eta \nabla\hat{\mu}^{k_t}(\theta_t;\mathcal{B}_t),
    \label{eq:ogd}
\end{equation}
where $\eta$ is the learning rate and $\hat{\mu}^{k_t}(\theta_t;\mathcal{B}_t)$ is the average empirical loss. 
When each data source corresponds to a distinct task or domain (\eg verbal reasoning \vs math reasoning), the loss function depends on the data source selected at the current step.
Thus the optimizer is an instance of \textit{online gradient descent (OGD)}~\cite{hazan_introduction_2016}. We analyze unprojected OGD under~\cref{as:domain}, which explicitly requires the iterates to remain in the optimization domain.

\textbf{Sampler.}
At the beginning of each step $t$, the sampler observes the algorithm trajectory $\mathcal{A}_{t-1}$ and the historical empirical losses $\{\hat{\mu}^{k_\tau}(\theta_\tau;\mathcal{B}_\tau),\tau=1,\dots,t-1\}$.
Selecting which data source to query is challenging for three reasons.
First, empirical mini-batch losses are stochastic estimates of the corresponding source-wise expected losses. 
Second, the sampling environment is non-stationary because the source-wise losses change as the optimizer updates the model parameters. 
Third, the sampler faces an exploration-exploitation tradeoff: it must collect enough observations from all sources to estimate their losses accurately, while also allocating samples to sources that appear to be currently worst-performing.

However, unlike the adversarial bandit settings, where losses could change arbitrarily over time, the variation in source-wise losses between consecutive steps is controlled by the OGD learning rate $\eta$ in~\cref{eq:ogd} together with the Lipschitz condition in~\cref{as:loss}.
Thus, the sampling problem is naturally related to slowly-varying non-stationary bandits~\citep{krishnamurthy_slowly-varying_2025}. 
While the classical upper-confidence-bound (UCB) algorithm suits  stationary environments, we use Discounted UCB (or DUCB)~\citep{kivinen_upper-confidence_2011}, a simple UCB variant with a discounting mechanism.

Specifically, let $\mathbbm{1}\{\cdot\}$ be the indicator function. For a discount factor $\beta\in[0,1)$, define the EMA estimate of the empirical loss of source $k$ at step $t$ as
\begin{equation}
    \hat{\mu}_t^k := \frac{1}{N_t^k(\beta)} \sum_{\tau=1}^t \mathbbm{1} \left\{ k_\tau = k \right\} \cdot \beta^{t-\tau} \hat{\mu}^k(\theta_\tau;\mathcal{B}_\tau), \quad
    N_t^k(\beta) := \sum_{\tau=1}^t \mathbbm{1} \left\{ k_\tau = k \right\} \cdot \beta^{t-\tau}.
    \label{eq:ema}
\end{equation}
Here, $N_t^k(\beta)$ reflects the \textit{effective sample size} of source $k$. Let $H(\beta):=1/(1-\beta)$ for $\beta\in[0,1)$ (known as the \textit{effective horizon} of a discounted sequence) and set $H(1)=T$ for undiscounted case.

During the first $K$ steps, DUCB performs initial exploration by selecting each source once, i.e., $k_t=t$ for $t=1,\dots,K$. Then for $t\ge K+1$, it selects the source $k_t$ as
\begin{equation}
    k_t \gets \arg\max_k ~ \big(\hat{\mu}_{t-1}^k + r_{t-1}^k\big), \quad
    r_{t-1}^k := 2 C \sqrt{\frac{2 \alpha \ln H(\beta)}{N_{t-1}^k(\beta)}},
    \label{eq:ducb}
\end{equation}
where $\alpha>0$ is a constant controlling the aggressiveness of exploration and $C$ bounds the loss as in~\cref{as:loss}. Equivalently, this sampling rule can be represented as a categorical distribution $k_t\sim\operatorname{Cat}(p_t)$ over $[K]$, with probabilities $p_{t,k}=\mathbbm{1}\{k=k_t\}$ for each $k\in[K]$.
We discuss the implementation details of DUCB using practical LLM post-training libraries in Appendix~\ref{sec:impl}.

Intuitively, the DUCB index simply couples the normalized EMA estimated empirical loss $\hat{\mu}_t^k$ with an optimistic confidence radius of $O(N_t^k(\beta)^{-1/2})$ for each $k\in[K]$. Due to the recursive identity of $N_t^k(\beta)=\beta N_{t-1}^k(\beta)+\mathbbm{1}\{k_t=k\}$, the effective sample size of each data source decays by a factor of $\beta$ whenever it is not selected at the current step.
This naturally reflects a loss of confidence in $\hat{\mu}_t^k$ over time, as the current model parameter $\theta_t$ drifts further away from the states in which the historical losses were observed.

\subsection{Main Theoretical Results}
\label{sec:method-theory}

Now we present the theoretical guarantees of DUCB-OGD. While the milestones of the proof are outlined below, we provide the complete version in Appendix~\ref{sec:proofs}.

\begin{theorem}[Dynamic Minimax Regret of DUCB-OGD]
    Suppose $\mu^k$ is convex in $\theta$ for all $k$.
    Then with $\eta=K^{-{1/4}}T^{-{3/4}}$ in \cref{eq:ogd} and $\beta=1-K^{-{1/2}}T^{-{1/2}}$ in \cref{eq:ema}, DUCB-OGD satisfies
    \begin{equation}
        \mathbb{E} \bigl[ R^{\rho,\gamma}(\operatorname{DUCB-OGD}_T) \bigr] \le \tilde O \bigl( K^\frac{1}{4} T^\frac{3}{4} \bigr),
    \end{equation}
    for any $\rho,\gamma\in[1-T^{-1},1]$.
    \label{th:minimax-regret}
\end{theorem}

In other words, the dynamic minimax regret achieves the rate of $\tilde{O}(K^\frac{1}{4}T^\frac{3}{4})$ in both undiscounted ($\rho=\gamma=1)$ and discounted scheme ($1-T^{-1}\le\rho,\gamma<1$).
To prove this, we first present the OGD optimization regret and the DUCB bandit regret respectively, and then combine them to obtain the overall dynamic minimax regret.

\begin{proposition}[Optimization Regret of OGD]
    Suppose $\mu^k$ is convex in $\theta$ for all $k$.
    OGD satisfies
    \begin{equation}
        \frac{T}{N_T(\rho)} \mathbb{E} \bigl[ R_o^\rho(\operatorname{OGD}_T) \bigr] \le O \bigl( \eta^{-1} \bigr) + O \bigl( \eta T \bigr),
        \label{eq:optim-regret}
    \end{equation}
    for any $\rho\in[1-T^{-1},1]$.
    \label{th:optim-regret}
\end{proposition}

\begin{remark}
    In the standard analysis of undiscounted and discounted OGD, the RHS of~\cref{eq:optim-regret} can be minimized by choosing $\eta=O(T^{-1/2})$, which gives the rate of $O(T^{1/2})$.
\end{remark}

The proof of~\cref{th:optim-regret} follows the standard OGD analysis for convex Lipschitz functions, with modifications to account for the stochasticity in the empirical loss observations. The full version is provided in Appendix~\ref{sec:proofs-optim}.
Crucially, while $\eta=O(T^\xi)$ for any $\xi\in(-1,0)$ in~\cref{eq:optim-regret} yields an optimization regret sublinear in $T$, we defer the choice of $\eta$ here since it dictates the rate that the sampling environment changes for DUCB, dramatically affecting its behavior. Instead, we first present the bandit regret with also an $\eta$ factor, then jointly tune $\eta$ to optimize both regrets.

\begin{proposition}[Bandit Regret of DUCB]
    Suppose $\eta=O(T^\xi)$ for some $\xi\in[-3/2,0)$. Then with $\beta=1-\eta^{2/3}K^{-1/3}$, DUCB satisfies
    \begin{equation}
        \frac{T}{N_T(\gamma)} \mathbb{E} \bigl[ R_b^\gamma(\operatorname{DUCB}_T) \bigr] \le \tilde O \bigl( \eta^\frac{1}{3} K^\frac{1}{3} T \bigr),
    \end{equation}
    for any $\gamma\in[1-T^{-1},1]$ and any $K$ independent of $T$.
    \label{th:bandit-regret}
\end{proposition}

\begin{remark}
    Note that when $\gamma=1$, $R_b^1(\operatorname{DUCB})$ matches the standard regret of slowly-varying non-stationary bandits by~\citet{krishnamurthy_slowly-varying_2025}.
    Alternatively, let $V:=O(\eta T)$ be the \textit{total variation budget} of the sampling environment over $T$ steps. Then~\cref{th:bandit-regret} can be rewritten as $R_b^1(\operatorname{DUCB})\le\tilde{O}(K^\frac{1}{3}T^\frac{2}{3}V^\frac{1}{3})$, matching the result of~\citet{besbes_stochastic_2014}.
\end{remark}

We provide a proof outline of~\cref{th:bandit-regret}.
While standard DUCB analysis focuses on abruptly changing environments~\cite{kivinen_upper-confidence_2011}, we adapt the proof to our slowly-varying environment based on the following decomposition.
Denote $k_t^\star:=\argmax_k\mu^k(\theta_t)$. Let $\tilde{\mu}_t^k:=\mathbb{E}[\hat{\mu}_t^k|\mathcal{A}_t]$ be the expected EMA estimated loss of the source $k$ at step $t$.
Then by~\cref{def:optimization-bandit-regret}, we have $R_b^\gamma(\mathcal{A}_T)=\sum_{t=1}^T\gamma^{T-t}(\mu^{k_t^\star}(\theta_t)-\mu^{k_t}(\theta_t))$, and for any $t\in[T]$,
\begin{equation}
    \mu^{k_t^\star}(\theta_t) - \mu^{k_t}(\theta_t)
    = \underbrace{\bigl( \tilde\mu_{t-1}^{k_t^\star} - \tilde\mu_{t-1}^{k_t} \bigr)}_{\text{Term 1}}
    - \underbrace{\bigl( \tilde\mu_{t-1}^{k_t^\star} - \mu^{k_t^\star}(\theta_t) \bigr)}_{\text{Term 2}}
    + \underbrace{\bigl( \tilde\mu_{t-1}^{k_t} - \mu^{k_t}(\theta_t) \bigr)}_{\text{Term 3}}.
    \label{eq:loss-decomposition}
\end{equation}
We prove in~\cref{th:bandit-gap-ducb} in Appendix~\ref{sec:proofs-gap} that the DUCB update rule guarantees that Term~1 is bounded by $O(r_{t-1}^k)$ with \textit{high probability}, where $r_{t-1}^k$ is the confidence radius defined in~\cref{eq:ducb}. Terms 2 and 3 reflect the bias of the EMA estimated loss \wrt the true loss, which are shown to be bounded by $\tilde{O}(\eta G^2H(\beta))$ in~\cref{th:ema-bias} in Appendix~\ref{sec:proofs-bias}. The complete proof of~\cref{th:bandit-regret} is provided in Appendix~\ref{sec:proofs-bandit}.
Now we can prove~\cref{th:minimax-regret}.

\begin{proof}[Proof of~\cref{th:minimax-regret}]
    Recall that the optimization regret in~\cref{th:optim-regret} is sublinear only when $\eta=O(T^\xi)$ for some $\xi\in(-1,0)$, which satisfies the condition of~\cref{th:bandit-regret}.
    Then by~\cref{def:minimax-regret}, and using~\cref{th:optim-regret} and~\cref{th:bandit-regret}, we have
    \begin{equation}
        \mathbb{E} \bigl[ R^{\rho,\gamma}(\operatorname{DUCB-OGD}_T) \bigr] \le O \bigl( \eta^{-1} \bigr) + O \bigl( \eta T \bigr) + \tilde O \bigl( \eta^\frac{1}{3} K^\frac{1}{3} T \bigr).
    \end{equation}
    The RHS is minimized (order-wise) by taking $\eta=K^{-{1/4}}T^{-{3/4}}$, which concludes the proof.
\end{proof}

Further note that when $\rho=\gamma$, the average dynamic minimax regret $R^{\rho,\gamma}/T$ gives the minimax optimality gap measured by the EMA estimated worst-case loss. Then the following result is a direct consequence of~\cref{th:minimax-regret}.

\begin{corollary}
    Under the conditions of~\cref{th:minimax-regret}, the expected minimax optimality gap at the EMA-weighted iterate of $\operatorname{DUCB-OGD}_T$, for any discount factor $\gamma\in[1-T^{-1},1]$, satisfies
    \begin{equation}
        \mathbb{E} \left[ \max_{k\in[K]} ~ \mu^k(\tilde{\theta}_T^\gamma) \right] - \max_{k\in[K]} ~ \mu^k(\theta^\star) \le \tilde{O} \bigl( K^\frac{1}{4} T^{-\frac{1}{4}} \bigr),
    \end{equation}
    where $\tilde{\theta}_T^\gamma:=\sum_{t=1}^T\gamma^{T-t}\theta_t/N_T(\gamma)$.
    \label{th:ema-convergence}
\end{corollary}

Specifically when $\rho=\gamma=1$, the results in~\cref{th:optim-regret,th:bandit-regret} are shown to be optimal up to logarithmic factors. 
This follows from the lower bounds for the optimization regret and the dynamic bandit regret, by~\citet{hazan_introduction_2016} and~\citet{krishnamurthy_slowly-varying_2025} respectively, based on constructing a hard instance where the loss landscape is non-smooth, and the worst-performing data source alters periodically over time.
Then it immediately follows that, in the undiscounted setting, the bound in~\cref{th:minimax-regret} is also optimal.

\begin{theorem}[Lower Bound]
    For every admissible sampler-optimizer algorithm $\mathcal{A}_T$, there exsits a problem instance for which the undiscounted regret satisfies $\mathbb{E}[R^{1,1}(\mathcal{A}_T)]\ge\tilde{\Omega}(K^\frac{1}{4}T^\frac{3}{4})$.
    \label{th:lower-bound}
\end{theorem}

\subsection{Implications}

We highlight two implications following the theoretical analysis.

\textbf{Comparison to Pseudo-Regret GDRO.}
Other GDRO frameworks proposed by~\citet{zhang_stochastic_2023} and~\citet{soma_near-optimal_2025} adopt the notion of \textit{pseudo regret}, defined as $\tilde{R}(\mathcal{A}_T):=\max_k\sum_{t=1}^T\mu^k(\theta_t)-T\max_k\mu^k(\theta^\star)$, which can be bounded by $\tilde{O}(K^{1/2}T^{1/2})$. The pseudo regret is fundamentally easier to analyze than our dynamic regret in~\cref{def:undiscount-minimax-regret}, since it always holds that $\tilde{R}\le R$ (and $R\le R^{1,1}$ by~\cref{th:regret-decomposition}).
In real-world learning problems such as LLM post-training, the worst-case data source may vary as the training progresses, and dynamic regret captures this by tracking the gap \wrt the worst-case loss in each step (we provide an example in~\cref{sec:expr-compare}).
Moreover, pseudo regret controls only the uniform time-average of algorithm performance, while our dynamic-regret formulation yields guarantees for exponentially weighted trajectory averages emphasizing on later iterates as shown in~\cref{th:ema-convergence}. This better matches the practical deep learning workflows, where training dynamics are often monitored through EMA-style statistics and final checkpoints are selected from the later stage of model training.

\textbf{Scope of the Lower Bound.}
We further note that \cref{th:lower-bound} establishes optimality for the undiscounted objective under the instantaneous mini-batch-only feedback oracle, and does not imply optimality for discounted objectives or when richer feedback is available.
Indeed, \citet{zang_adaptive_2025} re-evaluate historically collected samples under the current model and obtain a dynamic-regret rate of $\tilde{O}(K^{1/2}T^{1/2})$. Yet this stronger oracle requires expensive model evaluations over a linearly growing replay-buffer, and a generalization step from replay-based empirical losses to population losses.

\section{Experiments}
\label{sec:expr}

\begin{table}[!t]
  \caption{Main results of different sampling algorithms on various LLM post-training objectives across different model scales and families. The primary metric is \textit{minimum} source-wise accuracy.}
  \label{tab:results-acc}
  \centering
  \begin{tabular}{lcccccc}
    \toprule
    \multirow{2}{*}[-0.5ex]{\textbf{Objective}} & \multicolumn{2}{c}{\textbf{Dataset}} & \multirow{2}{*}[-0.5ex]{\textbf{Model}} & \multirow{2}{*}[-0.5ex]{\textbf{Sampler}} & \multicolumn{2}{c}{\textbf{Accuracy}} \\
    \cmidrule(lr){2-3} \cmidrule(lr){6-7} & \textbf{Train} & \textbf{Test} & & & \textbf{\textit{Min}} & \textbf{Avg} \\
    \midrule
    \multirow{6}{*}{SFT} &
        \multirow{6}{*}{\makecell{Tulu-v2\\(subset)}} & \multirow{6}{*}{\makecell{GSM8k\\+\\IFEval}} &
            \multirow{3}{*}{\makecell{Qwen2.5-7B\\(Base)}} &
                Uniform & 27.15 & 39.36 \\ & & & &
                Exp3-style$^*$ & 27.73 & \textbf{40.72} \\ & & & &
                DUCB (Ours) & \textbf{28.52} & 37.89 \\
            \cmidrule(lr){4-7} & & &
            \multirow{3}{*}{\makecell{Llama3.1-8B\\(Base)}} &
                Uniform & 22.07 & 25.98 \\ & & & &
                Exp3-style$^*$ & 21.48 & \textbf{27.93} \\ & & & &
                DUCB (Ours) & \textbf{24.22} & 25.29 \\
    \midrule
    \multirow{8}{*}{DPO} &
        \multicolumn{2}{c}{\multirow{8}{*}{\makecell{HelpSteer3\\(subset)}}} &
            \multirow{4}{*}{\makecell{Qwen2.5-3B\\(Instruct)}} &
                Uniform & 76.76 & 80.53 \\ & & & &
                AMA-R$^1$ & 78.19 & 79.45 \\ & & & &
                AMA-S$^1$ & \textbf{80.25} & \textbf{82.35} \\ & & & &
                DUCB (Ours) & \textbf{79.33} & 81.59 \\
            \cmidrule(lr){4-7} & & &
            \multirow{4}{*}{\makecell{Llama3.2-3B\\(Instruct)}} &
                Uniform & 73.79 & 81.06 \\ & & & &
                AMA-R$^1$ & 74.22 & 78.00 \\ & & & &
                AMA-S$^1$ & \textbf{77.37} & 80.48 \\ & & & &
                DUCB (Ours) & \textbf{77.42} & \textbf{81.88} \\
    \midrule
    \multirow{8}{*}{GRPO} &
        \multicolumn{2}{c}{\multirow{4}{*}{Reasoning}} &
            \multirow{4}{*}{\makecell{Qwen2.5-1.5B\\(Instruct)}} &
                Uniform & 18.89 & 28.98 \\ & & & &
                Exp3-style$^*$ & 20.91 & 28.78 \\ & & & &
                DUMP$^2$ & 19.45 & \textbf{29.70} \\ & & & &
                DUCB (Ours) & \textbf{21.81} & 28.74 \\
        \cmidrule(lr){2-7} &
        \multicolumn{2}{c}{\multirow{4}{*}{CommonsenseQA}} &
            \multirow{4}{*}{\makecell{GLM4-9B\\(Instruct)}} &
                Uniform & 77.91 & 84.08 \\ & & & &
                Exp3-style$^*$ & 82.56 & 85.55 \\ & & & &
                DUMP$^2$ & 81.40 & 85.55 \\ & & & &
                DUCB (Ours) & \textbf{84.88} & \textbf{87.74} \\
    \bottomrule
    \addlinespace[2pt]
    \multicolumn{7}{@{}l}{\footnotesize \makecell[l]{
        $^1$\citet{corrado_automixalign_2025}, requires training $K$ specialist models;
        $^2$\citet{wang_dump_2025}.\\
        $^*$Sampler-only proxies of~\citet{shin_dynamixsft_2025,ramesh_multi-task_2026}.
    }} \\
  \end{tabular}
\end{table}

\subsection{Objectives and Datasets}
\label{sec:expr-objective}

We evaluate DUCB-OGD across three LLM post-training objectives:
\begin{itemize}[leftmargin=2em,topsep=0em,itemsep=0em]
    \item \textit{SFT:}
    We train base models for 32k samples on the \textit{Math} and \textit{Orca} sources of Tulu-v2~\citep{ivison_camels_2023} using TRL~\citep{von_werra_trl_2020}. We evaluate the minimum of GSM8k accuracy~\citep{cobbe_training_2021} and IFEval prompt accuracy~\citep{zhou_instruction-following_2023}.
    \item \textit{DPO:}
    We train instruction-tuned models for 16k samples on the \textit{Code}, \textit{Multilingual}, and \textit{Stem} sources of HelpSteer3~\citep{wang_helpsteer3_2025,wang_helpsteer3-preference_2025,wang_rlbff_2025} using TRL. The primary metric is minimum source-wise test preference accuracy.
    \item \textit{GRPO:}
    We consider (i) \textit{Reasoning}, comprising \textit{Arc}, \textit{Countdown}, and \textit{Zebra} as three sources~\citep{chollet_measure_2019,katz_seemingly_2025,lin_zebralogic_2025}, with a 10k-sample budget using Verl~\citep{sheng_hybridflow_2024}; and (ii) CommonsenseQA~\citep{talmor_commonsenseqa_2019}, partitioned into \textit{Social}, \textit{Physical}, \textit{Causal-Temporal}, and \textit{General} sources, with a 6.4k-sample budget using TRL. We report the minimum source-wise test accuracy.
\end{itemize}
For each objective, we train models of different sizes from the Qwen2.5~\citep{yang_qwen2_2024,team_qwen25_2024}, Qwen3~\citep{team_qwen3_2025}, Llama3~\citep{aimeta_llama_2024}, and GLM4~\citep{glm_chatglm_2024} families using different sampling algorithms.

\subsection{Robustness Across Post-Training Objectives and Model Families}
\label{sec:expr-results}

\Cref{tab:results-acc} reports minimum source-wise accuracy as the primary metric and unweighted average accuracy as a reference. All metrics are evaluated at the final training checkpoint. Additional models and metrics are reported in Appendix~\ref{sec:results}.

For SFT, we evaluate an adapted Exp3-style proxy based on DynamixSFT~\citep{shin_dynamixsft_2025}. A faithful implementation requires virtual parameter updates to estimate anticipated source-wise loss decreases and is not directly compatible with our training pipeline, and our proxy instead uses observed source-wise loss decreases. DUCB achieves the highest minimum accuracy for both models, whereas the Exp3-style sampler obtains higher average accuracy, illustrating the robustness-average performance trade-off.

For DPO, we compare against faithful implementations of AutoMixAlign~\cite{corrado_automixalign_2025}, including both AMA-R and AMA-S. AutoMixAlign uses source-specific specialist models to compute excess-loss-based sampling signals, which requires training $K$ additional specialists before the generalist model. DUCB achieves competitive or better worst-source performance without this auxiliary training cost.
These results further support the robustness of DUCB under a substantially lighter feedback and computational requirement.
Specifically, across three Qwen2.5-3B runs, DUCB obtains $79.33\pm0.62$ minimum accuracy versus $76.76\pm1.65$ for Uniform, corresponding to a paired gain of $2.57\pm1.04$ percentage points.

For GRPO, we compare against an Exp3-style sampler-only proxy of MT-GRPO~\cite{ramesh_multi-task_2026} and a faithful implementation of DUMP~\cite{wang_dump_2025} which adopts a switching-window UCB sampler. We use the source-wise training reward with improvement correction and the training advantage as the sampling signal, respectively. We omit additional components of MT-GRPO such as rollout filtering and only implement the sampler in order to maintain uniformity with other samplers and compatibility between the TRL and Verl libraries.
The results show clear margins of DUCB over other uniform and adaptive sampling baselines on the \textit{minimum} accuracies across both datasets.

Across the six main settings, DUCB achieves the highest minimum accuracy in five and remains competitive with AMA-S in the remaining DPO setting.
Detailed discussions of the results are also provided in Appendix~\ref{sec:results}.

\subsection{Computation Efficiency and Dynamic Worst-Source Adaptation}
\label{sec:expr-compare}

We specifically discuss the following two aspects of the proposed DUCB-OGD algorithm compared with other evaluated baselines.

\paragraph{Training Overhead.}
We measure per-step runtime across three Qwen2.5-3B DPO runs, with full results reported in Appendix~\ref{sec:results-overhead}. DUCB requires $21.94\pm1.14$ seconds per step, compared with $21.90\pm1.29$ for Uniform; the difference is small relative to the observed run-to-run variation. Its sampler adds no auxiliary model evaluations and maintains only $O(K)$ scalar state.
By comparison, AMA additionally trains $K$ source-specific specialists before training the generalist model. Thus, DUCB introduces little measurable online overhead beyond the underlying LLM training step and avoids AMA's auxiliary specialist-training cost.

\paragraph{Adaptation to Evolving Source Difficulty.}
We further compare DUCB's instantaneous tracking with the cumulative-reward updates of Exp3 on the \textit{Arc} and \textit{Countdown} sources of the Reasoning task in~\cref{fig:compare}. Arc initially has higher reward than Countdown, but this ordering later reverses, making Arc the new worst-performing source. DUCB shifts its sampling priority toward Arc shortly after the crossover, whereas Exp3 continues to favor Countdown until later observations outweigh its accumulated history.
This representative example, together with the switching frequencies in~\cref{tab:prelim-objective}, illustrates DUCB's ability to adapt to changing source difficulty. Because instantaneous tracking is also more sensitive to short-term reward variation, we use DEMA smoothing and forced exploration as practical stabilizers, detailed in Appendix~\ref{sec:impl}.

\begin{figure}[t]
    \centering
    \begin{subfigure}[t]{0.28\columnwidth}
        \centering
        \includegraphics[width=\linewidth]{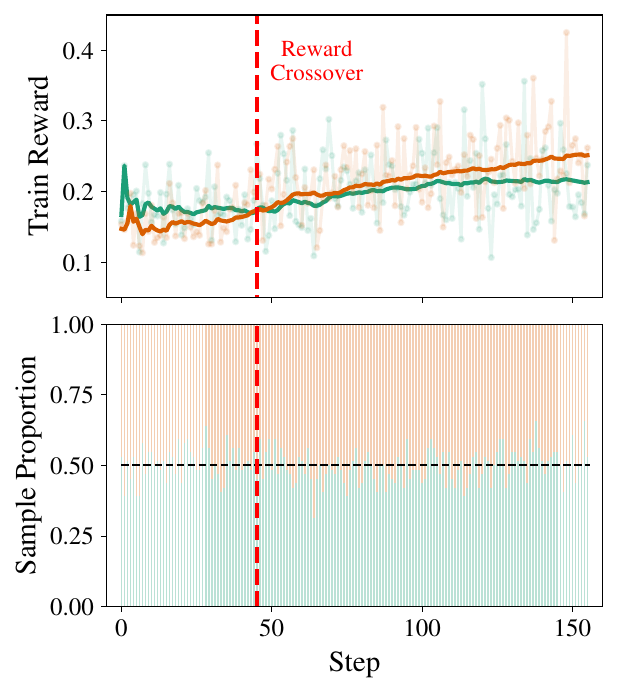}
        \caption{Uniform}
        \label{fig:compare-uniform}
    \end{subfigure}
    \hfill
    \begin{subfigure}[t]{0.28\columnwidth}
        \centering
        \includegraphics[width=\linewidth]{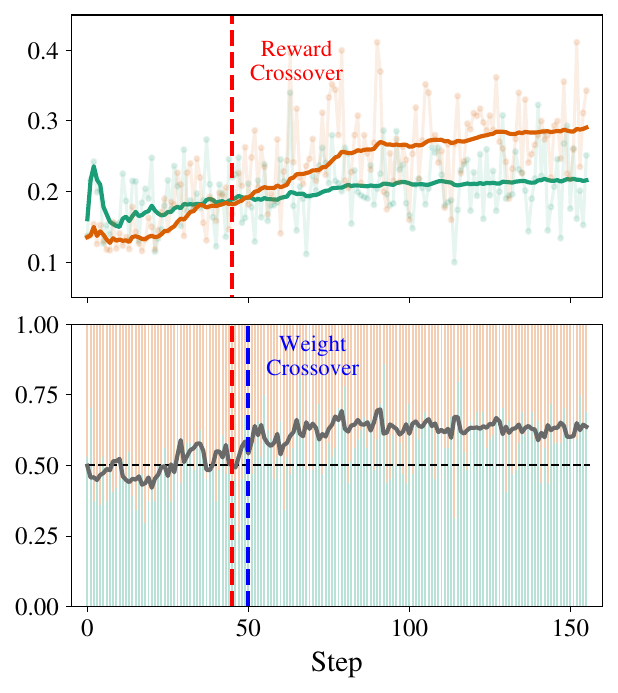}
        \caption{DUCB}
        \label{fig:compare-ducb}
    \end{subfigure}
    \hfill
    \begin{subfigure}[t]{0.28\columnwidth}
        \centering
        \includegraphics[width=\linewidth]{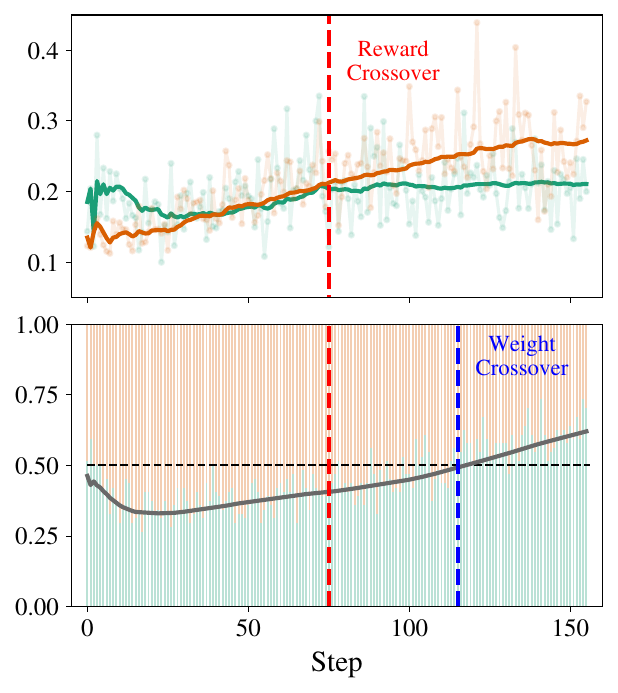}
        \caption{Exp3}
        \label{fig:compare-exp3}
    \end{subfigure}
    \begin{subfigure}[t]{0.14\columnwidth}
        \centering
        \includegraphics[width=\linewidth]{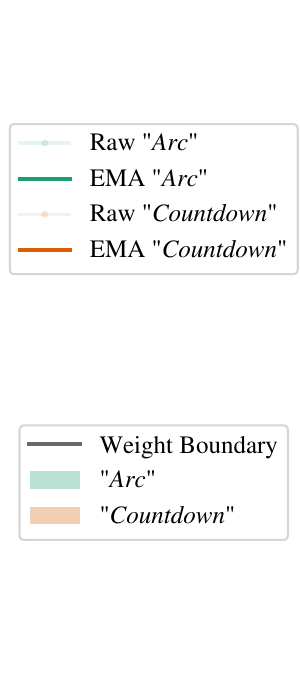}
    \end{subfigure}
    \caption{GRPO training rewards and sampling weights for uniform sampling, DUCB (guided by dynamic regret), and Exp3 (guided by pseudo regret) on two data sources (\textit{Arc} and \textit{Countdown}).}
    \label{fig:compare}
\end{figure}

\section{Conclusion}
\label{sec:conclusion}

We study dynamic minimax regret for robust multi-source LLM post-training under instantaneous mini-batch-only bandit feedback. We propose DUCB-OGD, which couples discounted UCB sampling with online gradient descent and requires neither historical re-evaluation nor auxiliary model evaluations.
Across SFT, DPO, and GRPO, DUCB-OGD achieves the highest minimum source-wise accuracy in the majority of the settings and remains competitive with contemporary baselines but avoids expensive model re-evaluations.
Frequent worst-source switching is observed across all three objectives, while DUCB's per-step runtime remains close to Uniform, supporting the relevance and practical efficiency of instantaneous dynamic tracking.

It is shown that DUCB-OGD is most applicable when worst-source robustness is the primary objective, source difficulty changes during training, and only instantaneous mini-batch feedback is available or affordable due to practical constraints. Trade-offs may occur between worst-source and average performances. The guarantees involving optimization also do not cover non-convex loss landscapes.
Our comparisons include both faithful baseline implementations and explicitly adapted sampler-only proxies, which avoid objective-specific training modifications. The evaluation is limited to 0.5B-9B models and unimodal datasets with relatively few sources. The performance on larger models and datasets with more complicated distributional characteristics remains to be studied.

\subsubsection*{Acknowledgments}
This work was supported in part by NSF grant CNS-2450567 and EEC-2133516, NSF grant CNS-2038984 and corresponding support from the Federal Highway Administration (FHA), MediaTek Inc USA, NSF grant CNS-2148128 and by funds from federal agency and industry partners as specified in the Resilient \& Intelligent NextG Systems (RINGS) program, and ARO grant W911NF2210031.
We gratefully acknowledge use of the research computing resources of the Empire AI Consortium, Inc, with support from Empire State Development of the State of New York, the Simons Foundation, and the Secunda Family Foundation.

\subsection*{AI Use Statement}

In this work, we did not use generative AI tools to generate datasets, develop theoretical models or methods, formulate or prove mathematical claims, design experiments, implement methods, process data, or interpret experimental results.
We used generative AI tools only to polish writing and improve readability. We reviewed the accuracy and consistency of all AI-assisted work, including revised text, equations, and descriptions of experimental results.
We take responsibility for the final content of this work, including all text, claims, and artifacts produced with the aid of generative AI.

\subsection*{Reproducibility statement}

We provide the complete DUCB-OGD algorithm and theoretical assumptions in the main text, with full proofs in Appendix~\ref{sec:proofs}. Appendix~\ref{sec:impl} and~\ref{sec:results} document the implementation, practical stabilizers, key training configurations, and hyperparameters.
The datasets, pretrained models, sample budgets, evaluation metrics, final-checkpoint protocol, multi-seed results, and hardware used for runtime measurements are described in~\cref{sec:expr} and Appendix~\ref{sec:results}.
All datasets and model checkpoints are publicly available and cited. We intend to release the training and evaluation code, full configurations, and scripts upon publication.

\bibliography{references}
\bibliographystyle{iclr2027_conference}

\newpage
\appendix
\section*{Appendix}

\section{Related Work}
\label{sec:review}

\paragraph{Distributionally Robust Optimization.}
Distributionally Robust Optimization studies learning algorithms that optimize model performance under worst-case shifts or uncertainty in the data distribution~\cite{hu_does_2018,duchi_learning_2021,jin_non-convex_2021,agarwal_minimax_2022}. 
Based on this formulation, GDRO explicitly minimizes the worst-case loss across predefined distribution groups~\cite{sagawa_distributionally_2020,haghtalab_-demand_2022,ghosal_distributionally_2023,zhang_stochastic_2023}. Recent theoretical work formulates GDRO as a stochastic convex-concave saddle-point problem and studies algorithms that combine data sampling with parameter optimization. For example,  \citet{zhang_stochastic_2023} use Exp3-style sampling ~\cite{auer_finite-time_2002} to track the \textit{fixed} worst-performing data source \textit{in hindsight}, and  bound the minimax optimality gap for the time-averaged model parameters, but does not account for the evolving nature of the training process.
Closest to our setting, \citet{zang_adaptive_2025} propose to track the \textit{adaptive} sequence of worst-performing data sources using UCB-style sampling~\cite{auer_finite-time_2002}. However, their approach requires repeatedly re-evaluating the current model on previously collected samples to maintain up-to-date source-wise loss estimates. This introduces substantial computational overhead and requires intrusive modifications to the standard training pipeline, limiting its practibility and compatibility with modern LLM training libraries.

\paragraph{Non-stationary Bandits.}
Multi-armed bandit algorithms provide a natural framework for adaptive data-source selection. 
Classical stochastic bandits assume fixed reward distributions~\cite{auer_finite-time_2002}, while adversarial bandits allow arbitrary reward sequences~\cite{auer_nonstochastic_2002}. 
Non-stationary bandits interpolate between these extremes by imposing structure on how rewards evolve over time~\cite{kivinen_upper-confidence_2011,besbes_stochastic_2014,krishnamurthy_slowly-varying_2025}. 
Discounted-UCB methods were originally designed for environments with abrupt changes~\cite{kivinen_upper-confidence_2011}, while variation-budget and slowly-varying models provide alternative ways to quantify non-stationarity~\cite{besbes_stochastic_2014,krishnamurthy_slowly-varying_2025}. 
Our setting differs from standard non-stationary bandits because the non-stationarity is endogenous: source-wise losses change because the optimizer updates the model parameters. 
Thus, the sampler and optimizer are coupled. 
We exploit this structure by combining discounted bandit estimation with online gradient descent and analyzing the resulting sampler-optimizer behavior through dynamic minimax regret.

\paragraph{Adaptive Sampling for LLM Training.}
Adaptive sampling has become increasingly important in LLM training and post-training because data sources vary widely in quality, difficulty, and relevance. 
Some methods estimate mixture weights offline using data-value models, filtering scores, or proxy evaluations~\cite{lin_data-efficient_2024,lv_data-efficient_2025,wang_ultra-fineweb_2025,sachdeva_how_2026}. 
Others update sampling probabilities online using training losses, gradient statistics, task difficulty, or model feedback~\cite{shin_dynamixsft_2025,koh_msft_2026,chen_self-evolving_2025,wang_dump_2025,sun_improving_2026,panaganti_group_2026}. 
Recent work on multi-task post-training also studies how to balance heterogeneous reasoning, coding, safety, and instruction-following data so that easier or larger sources do not dominate the training~\cite{li_pike_2026,ramesh_multi-task_2026,wu_imbalanced_2026,panaganti_group_2026}. 
DRO-inspired objectives have also been explored for improving worst-case performance across groups or domains~\cite{oren_distributionally_2019,corrado_automixalign_2025,wang_hbo_2026,shao_baldro_2026}.
However, existing approaches typically optimize static or heuristic objectives and lack a principled dynamic minimax regret formulation that explicitly accounts for the non-stationary nature of LLM training while providing provable theoretical guarantees.

\newpage
\section{Implementation Details}
\label{sec:impl}

\begin{algorithm}
    \caption{Adaptive Sampling for Robust Multi-Source LLM Post-Training.}
    \label{alg:outline}
    \begin{algorithmic}[1]
        \Require Data sources $\{\mathcal{D}^k\}_{k=1}^K$, initial parameter $\theta_1$, horizon $T$
        \State Set $\beta \gets 1 - K^{-{1/2}}T^{-{1/2}}$
        \State Initialize $\hat{\mu}_t^k, N_t^k(\beta) \gets 0$ for all $k\in[K]$
        \For{$t=1,\dots,T$}
            \State If $t\le K$, set $k_t\gets t$; otherwise, generate sampling probabilities
            \[
            p_t \gets \textsc{Sampler}(\{\mathcal{D}^k, \hat{\mu}_{t-1}^k, N_{t-1}^k(\beta)\}_{k=1}^K)
            \]
            \State Sample the mini-batch
            \[
            \mathcal{B}_t \sim \big\{ X \sim \mathcal{D}^{k_{t,i}} : k_{t,i} \sim \operatorname{Cat}(p_t), i = 1,2,\dots \bigr\}
            \]
            \State Parameter update
            \[
            \theta_{t+1} \gets \textsc{Optimizer}(\theta_t, \mathcal{B}_t)
            \]
            \State Update $\hat{\mu}_t^k$ and $N_t^k(\beta)$ by~\cref{eq:ema} for all $k\in[K]$
        \EndFor
        \State \Return $\theta_{T+1}$
    \end{algorithmic}
\end{algorithm}

The training process is outlined in~\cref{alg:outline}.
We further adopt the following techniques during implementation.

\paragraph{Batch Sampling.}
Instead of selecting only one data source in each step, we propose a sampling version of DUCB, defined as
\begin{equation*}
    k_t \gets \argmax_{k\in[K]} ~ \hat{\mu}_{t-1}^k + \phi_{t-1}^k, \quad \phi_{t-1}^k \sim \mathcal{N} \Bigg( 2 r_{t-1}^k, \frac{C^2}{2 N_{t-1}^k(\beta)} \Bigg).
\end{equation*}
Due to the stochasticity of $\phi_{t-1}^k$, the selection rule induces a \textit{non-degenerate} categorical distribution over $[K]$ with probabilities $p_t$ \st $k_t\sim\operatorname{Cat}(p_t)$, as opposed to the \textit{deterministic} rule in~\cref{eq:ducb}. Therefore, every mini-batch may contain a mixture of samples from multiple data sources shown by~\cref{alg:outline}.
This is especially important when using a large batch size, since training on a large amount of samples from a single data source increases the risk of catastrophic forgetting or mode collapse.
We further prove in Appendix~\ref{sec:proofs-sducb} that the bandit regret of the sampling version of DUCB is only a constant factor worse than that of the standard DUCB in~\cref{th:bandit-regret}.
Since $p_t$ does not admit a closed-form expression, we approximate it using the Gauss-Hermite quadrature~\citep{butler_computationally_1982}. At the end of each training step, we update $\hat{\mu}_t^k$ and $N_t^k(\beta)$ using the empirical loss and sample proportion in the mini-batch for each data source.

\paragraph{Double-EMA.}
Traditional EMA is known to have a delayed response to changing inputs due to the discounting mechanism. Thus we employ \textit{double-EMA}~\cite{mulloy_smoothing_1994} to track the training losses in the implementation of all samplers, defined as
\begin{equation*}
    \textsc{DoubleEma}(x) \gets 2 \cdot \textsc{Ema}(x) - \textsc{Ema}(\textsc{Ema}(x)).
\end{equation*}
As shown in~\cref{fig:compare}, the DEMA reward is able to closely track the trend of raw reward while providing sufficient denoising ability.

\paragraph{Enforced Exploration.}
In the early stage of training, the sampling weights of DUCB can be affected by noisy initial observations and experience fluctuations. Inspired by Exp3 (and its variants), we augment the sampling weights by
\begin{equation*}
    p_t^k \gets (1 - \beta^t) p_t^k + \frac{\beta^t}{K}, \quad \forall k \in [K].
\end{equation*}
This enforces uniform exploration across all data sources in the early stage of training, and exponentially decays as the training progresses and becomes more stable.
It is worth noting that the extra regret incurred by the enforced exploration is \textit{at most} $O(H(\beta))$, which is dominated by the standard DUCB regret given the choice of $\beta$ in~\cref{th:bandit-regret}.

\newpage
\section{Proofs}
\label{sec:proofs}

\subsection{Proof of \texorpdfstring{\cref{th:regret-decomposition}}{Optimization Regret}}
\label{sec:proofs-decom}

\begin{proof}[Proof of \cref{th:regret-decomposition}]
    Recall that $k_t^\star=\max_k\mu^k(\theta)$.
    By~\cref{def:optimization-bandit-regret},
    \begin{equation*}
        \begin{aligned}
            R_o^1(\mathcal{A}_T) + R_b^1(\mathcal{A}_T)
            &= \sum_{t=1}^T \Big( \mu^{k_t}(\theta_t) - \mu^{k_t}(\theta^\star) \Big) + \sum_{t=1}^T \Big( \mu^{k_t^\star}(\theta_t) - \mu^{k_t}(\theta_t) \Big) \\
            &= \sum_{t=1}^T \Big( \mu^{k_t^\star}(\theta_t) - \mu^{k_t}(\theta^\star) \Big) \\
            &\ge \sum_{t=1}^T \mu^{k_t^\star}(\theta_t) - T \max_{k\in[K]} ~ \mu^k(\theta^\star) \\
        \end{aligned},
    \end{equation*}
    where the last inequality follows from the fact that $\max_k\mu^k(\theta)\ge\mu^{k_t}(\theta)$ for all $\theta$ and $t$.
    Further recalling~\cref{def:undiscount-minimax-regret} concludes the proof.
\end{proof}

\subsection{Proof of \texorpdfstring{\cref{th:optim-regret}}{Optimization Regret}}
\label{sec:proofs-optim}

The proof follows from standard results for undiscounted and discounted OGD.
Suppose a player of an online game observes a sequence of losses $\{f_t(\theta_t),t=1,\dots,T\}$ for $T$ steps, and updates the parameter using OGD defined in~\cref{eq:ogd}.
Then the following lemmas hold. Note that we defer the choice of $\eta$ in both cases.

\begin{lemma}[Undiscounted OGD~\citep{hazan_introduction_2016}]
    For any $\theta\in\mathbf{\Theta}$, OGD satisfies
    \begin{equation*}
        \sum_{t=1}^T f_t(\theta_t) - \sum_{t=1}^T f_t(\theta) \le \frac{D^2}{\eta} + \eta G^2 T.
    \end{equation*}
    \label{th:optim-regret-undiscount-ogd}
\end{lemma}

\begin{lemma}[Discounted OGD~\citep{yang_discounted_2026}]
    Let $\rho\in[0,1)$ be a discount factor.
    For any $\theta\in\mathbf{\Theta}$, OGD satisfies
    \begin{equation*}
        \sum_{t=1}^T \rho^{T-t} f_t(\theta_t) - \sum_{t=1}^T \rho^{T-t} f_t(\theta) \le \Biggl( \frac{(1 - \rho) D^2}{\eta} + \frac{\eta G^2}{2} \Biggr) N_T(\rho).
    \end{equation*}
    \label{th:optim-regret-discount-ogd}
\end{lemma}

Now we can prove the optimization regret.

\begin{proof}[Proof of~\cref{th:optim-regret}]
    Let $f_t(\theta_t)=\hat{\mu}^{k_t}(\theta_t;\mathcal{B}_t)$.
    Recall that $X\sim\mathcal{D}^{k_t}$ \iid for all $X\in\mathcal{B}_t$ according to~\cref{as:iid}, \ie
    \begin{equation*}
        \mathbb{E} \bigl[ \hat{\mu}^{k_t}(\theta_t;\mathcal{B}_t) \big| k_t, \theta_t \bigr] = \mu^{k_t}(\theta_t).
    \end{equation*}
    Since $\theta^\star\in\mathbf{\Theta}$ is fixed, we have, by~\cref{th:optim-regret-undiscount-ogd},
    \begin{equation*}
        \begin{aligned}
            \mathbb{E} \bigl[ R_o^1(\mathcal{A}_T) \bigr]
            &= \sum_{t=1}^T \mathbb{E} \bigl[ \mu^{k_t}(\theta_t) - \mu^{k_t}(\theta^\star) \bigr] \\
            &= \sum_{t=1}^T \mathbb{E} \bigl[ \mathbb{E} \bigl[ \hat{\mu}^{k_t}(\theta_t;\mathcal{B}_t) - \hat{\mu}^{k_t}(\theta^\star;\mathcal{B}_t) \big| k_t, \theta_t \bigr] \bigr] \\
            &= \mathbb{E} \Biggl[ \sum_{t=1}^T \hat{\mu}^{k_t}(\theta_t;\mathcal{B}_t) - \sum_{t=1}^T \hat{\mu}^{k_t}(\theta^\star;\mathcal{B}_t) \Biggr] \\
            &\le \frac{D^2}{\eta} + \eta G^2 T
        \end{aligned},
    \end{equation*}
    and similarly by~\cref{th:optim-regret-discount-ogd},
    \begin{equation*}
        \mathbb{E} \bigl[ R_o^\rho(\mathcal{A}_T) \bigr] \le \Biggl( \frac{(1 - \rho) D^2}{\eta} + \frac{\eta G^2}{2} \Biggr) N_T(\rho), \quad \forall \rho \in [0,1).
    \end{equation*}
    Further note that $N_T(1)=T$. Combining together, we have
    \begin{equation*}
        \frac{T}{N_T(\rho)} \mathbb{E} \bigl[ R_o^\rho(\mathcal{A}_T) \bigr] \le O \bigl( \eta^{-1} (1 - \rho) T\bigr) + O \bigl( \eta T \bigr) \le O \bigl( \eta^{-1} \bigr) + O \bigl( \eta T \bigr),
    \end{equation*}
    for any $\rho\in[1-T^{-1},1]$, which concludes the proof.
\end{proof}

\subsection{Proof of \texorpdfstring{\cref{th:bandit-regret}}{Dynamic Bandit Regret}}
\label{sec:proofs-bandit}

We first present the following results, proved respectively in~\cref{sec:proofs-gap,sec:proofs-bias}.
As opposed to stationary bandits, the explorations for non-stationary bandits conducted in distant steps gradually expire due to changing environments. Therefore, continuous re-exploration is needed to ensure reliable concentration results, \ie we need to keep $N_t^k(\beta)$ above some threshold $A>0$.

\begin{proposition}[Bandit Gap of DUCB]
    Suppose DUCB picks $k_t$ in step $t\ge K+1$ as in~\cref{eq:ducb}. Then for any $A>0$,
    \begin{equation*}
        \Pr \left[ \tilde\mu_{t-1}^{k_t^\star} - \tilde\mu_{t-1}^{k_t} > 2 r_t^{k_t} \sqrt{1 + A^{-1}}, N_{t-1}^{k_t}(\beta) \ge A \right] \le 2 H(\beta)^{-\alpha}.
    \end{equation*}
    \label{th:bandit-gap-ducb}
\end{proposition}

\begin{proposition}[EMA Bias]
    Suppose $\{\theta_t,t=1,\dots,T\}$ is generated by OGD with learning rate $\eta$ as in~\cref{eq:ogd}. Then for any $t\in[T]$ and $A>0$ \st $N_t^k(\beta)\ge A$,
    \begin{equation*}
        \left| \tilde\mu_t^k - \mu^k(\theta_{t+1}) \right| \le \eta G^2 H(\beta) \ln \left( \frac{2 C e}{\eta A G^2} \right),
    \end{equation*}
    where $\beta\ge1/2$ and $H(\beta)$ is the effective horizon as in~\cref{eq:ducb}.
    \label{th:ema-bias}
\end{proposition}

We further state following lemmas.

\begin{lemma}
    For any $\tau\in[T]$ and constant $A>0$,
    \begin{equation*}
        \sum_{t=K+1}^T \mathbbm{1} \left\{ k_t = k, N_{t-1}^k(\beta) < A \right\} \cdot \gamma^{T-t} \le \lceil T / \tau \rceil N_{\beta^{-\tau}A}^k(\gamma).
    \end{equation*}
    \label{th:reexplore}
\end{lemma}

\begin{proof}
    We first write
    \begin{equation*}
        \sum_{t=K+1}^T \mathbbm{1} \left\{ k_t = k, N_{t-1}^k(\beta) < A \right\} \cdot \gamma^{T-t} \le \sum_{m=1}^{\lceil{T/\tau}\rceil} \sum_{t=(m-1)\tau}^{m\tau} \mathbbm{1} \left\{ k_t = k, N_{t-1}^k(\beta) < A \right\} \cdot \gamma^{T-t}.
    \end{equation*}
    By~\citet{kivinen_upper-confidence_2011} (Lemma 25 and Corollary 26),
    \begin{equation*}
        \sum_{t=(m-1)\tau}^{m\tau} \mathbbm{1} \left\{ k_t = k, N_{t-1}^k(\beta) < A \right\} \le \beta^{-\tau} A.
    \end{equation*}
    Further note that
    \begin{equation*}
        \begin{aligned}
            \sum_{t=(m-1)\tau}^{m\tau} &\mathbbm{1} \left\{ k_t = k, N_{t-1}^k(\beta) < A \right\} \cdot \gamma^{T-t} \\
            &= \gamma^{T-m\tau} \sum_{t=(m-1)\tau}^{m\tau} \mathbbm{1} \left\{ k_t = k, N_{t-1}^k(\beta) < A \right\} \cdot \gamma^{m\tau-t} \\
            &\le \gamma^{T-m\tau} N_{\beta^{-\tau}A}^k(\gamma) \\
            &\le N_{\beta^{-\tau}A}^k(\gamma) \\
        \end{aligned}.
    \end{equation*}
    Plugging back to the first equation concludes the proof.
\end{proof}

\begin{lemma}
    For all $T\ge K+1$,
    \begin{equation*}
        \sum_{t=K+1}^T \frac{\mathbbm{1} \left\{ k_t = k \right\} \cdot \gamma^{T-t}}{\sqrt{N_t^k(\beta)}} \le 2 N_T^k(\beta)^\frac{1}{2} + N_T^k(\gamma) H(\beta)^{-\frac{1}{2}}.
    \end{equation*}
    \label{th:discount-sum}
\end{lemma}

\begin{proof}
    Recall that
    \begin{equation*}
        \mathbbm{1}\{k_t=k\} = N_t^k(\beta) - \beta N_{t-1}^k(\beta) = N_t^k(\beta) - N_{t-1}^k(\beta) + (1 - \beta) N_{t-1}^k(\beta).
    \end{equation*}
    Thus we can write
    \begin{equation*}
        \begin{aligned}
            \frac{\mathbbm{1} \left\{ k_t = k \right\}}{\sqrt{N_t^k(\beta)}}
            &= \frac{\mathbbm{1} \left\{ k_t = k \right\}^2}{\sqrt{N_t^k(\beta)}} \\
            &= \mathbbm{1} \left\{ k_t = k \right\} \cdot \frac{N_t^k(\beta) - N_{t-1}^k(\beta)}{\sqrt{N_t^k(\beta)}} + \mathbbm{1} \left\{ k_t = k \right\} \cdot \frac{(1 - \beta) N_{t-1}^k(\beta)}{\sqrt{N_t^k(\beta)}}
        \end{aligned}.
    \end{equation*}
    For the first term, note that
    \begin{equation*}
        \frac{N_t^k(\beta) - N_{t-1}^k(\beta)}{\sqrt{N_t^k(\beta)}} \le 2 \Bigl( \sqrt{N_t^k(\beta)} - \sqrt{N_{t-1}^k(\beta)} \Bigr).
    \end{equation*}
    Since $\gamma\le1$ and $N_t^k(\beta)\ge N_{t-1}^k(\beta)$ whenever $k_t=k$, we have
    \begin{equation*}
        \begin{aligned}
            \sum_{t=K+1}^T \mathbbm{1} \left\{ k_t = k \right\} \cdot \gamma^{T-t} \Biggl( \frac{N_t^k(\beta) - N_{t-1}^k(\beta)}{\sqrt{N_t^k(\beta)}} \Biggr)
            &\le \sum_{t=K+1}^T \Biggl( \frac{N_t^k(\beta) - N_{t-1}^k(\beta)}{\sqrt{N_t^k(\beta)}} \Biggr) \\
            &\le 2 \sum_{t=K+1}^T \Bigl( \sqrt{N_t^k(\beta)} - \sqrt{N_{t-1}^k(\beta)} \Bigr) \\
            &\le 2 N_T^k(\beta)^\frac{1}{2} \\
        \end{aligned}.
    \end{equation*}
    For the second term, note that $N_t^k(\beta)\le H(\beta)=1/(1-\beta)$ for all $t$. Thus
    \begin{equation*}
        \begin{aligned}
            \sum_{t=K+1}^T \mathbbm{1} \left\{ k_t = k \right\} \cdot \gamma^{T-t} \Bigg( \frac{(1 - \beta) N_{t-1}^k(\beta)}{\sqrt{N_t^k(\beta)}} \Biggr)
            &\le (1 - \beta) \sum_{t=K+1}^T \gamma^{T-t} \sqrt{N_{t-1}^k(\beta)} \\
            &\le (1 - \beta) N_T^k(\beta)^\frac{1}{2} \sum_{t=K+1}^T \gamma^{T-t} \\
            &\le N_T^k(\gamma) H(\beta)^{-\frac{1}{2}} \\
        \end{aligned}.
    \end{equation*}
    Combining together gives the statement of the lemma.
\end{proof}

Now we can prove the dynamic bandit regret. In fact, we prove a more explicit version as follows.

\begin{proposition}[Unnormalized Dynamic Bandit Regret of DUCB]
    For any $\beta\ge1/2$, DUCB satisfies
    \begin{equation*}
        \begin{aligned}
            \mathbb{E} \bigl[ R_b^\gamma(\operatorname{DUCB}_T) \bigr]
            \le&~ O \bigl( K T H(\beta)^{-1} \bigr) + \tilde{O} \bigl( K^\frac{1}{2} H(\beta)^\frac{1}{2} \bigr) \\
            &~ + \tilde{O} \bigl( N_T(\gamma) H(\beta)^{-\frac{1}{2}} \bigr) + \tilde{O} \bigl( \eta N_T(\gamma) H(\beta) \bigr).
        \end{aligned}
    \end{equation*}
    \label{th:bandit-regret-ducb-unnorm}
\end{proposition}

\begin{proof}
    Let $A\ge1/3$ be a constant. We start by writing
    \begin{equation*}
        \begin{aligned}
            R_b^\gamma(\operatorname{DUCB}_T)
            =&~ \sum_{t=1}^T \gamma^{T-t} \Bigl( \mu^{k_t^\star}(\theta_t) - \mu^{k_t}(\theta_t) \Bigr) \\
            \le&~ \gamma^{T-K} K C + C \sum_{t=K+1}^T \mathbbm{1} \left\{ N_{t-1}^{k_t}(\beta) < A \right\} \cdot \gamma^{T-t} \\
            &~ + \sum_{t=K+1}^T \mathbbm{1} \left\{ N_{t-1}^{k_t}(\beta) \ge A \right\} \cdot \gamma^{T-t} \Bigl( \mu^{k_t^\star}(\theta_t) - \mu^{k_t}(\theta_t) \Bigr) \\
        \end{aligned}.
    \end{equation*}
    Denote $N_T(\cdot)=\sum_{k=1}^KN_t^k(\cdot)$. By~\cref{th:reexplore}, the first sum writes
    \begin{equation*}
        \begin{aligned}
            C \sum_{t=K+1}^T \mathbbm{1} \left\{ N_{t-1}^{k_t}(\beta) < A \right\} \cdot \gamma^{T-t}
            &\le C \sum_{k=1}^K \sum_{t=K+1}^T \mathbbm{1} \left\{ k_t = k, N_{t-1}^k(\beta) < A \right\} \cdot \gamma^{T-t} \\
            &\le C K \lceil T / \tau \rceil N_{\beta^{-\tau}A}(\gamma) \\
        \end{aligned}.
    \end{equation*}
    Letting $\tau=N_T(\beta)$, we have
    \begin{equation*}
        \beta^{-\tau} = \beta^{-N_T(\beta)} \le
        \begin{cases}
            \begin{aligned}
                & \beta^{-1/(1-\beta)} \le 4 & \quad \frac{1}{2} \le \beta < 1 \\
                & 1 & \quad \beta = 1 \\
            \end{aligned}
        \end{cases}.
    \end{equation*}
    Thus $N_{\beta^{-\tau}A}(\gamma)\le\beta^{-\tau}A\le4A$ for all $1/2\le\beta\le1$, and
    \begin{equation*}
        C \sum_{t=K+1}^T \mathbbm{1} \left\{ N_{t-1}^{k_t}(\beta) < A(\beta) \right\} \cdot \gamma^{T-t} \le 4 A C K \bigl( T N_T(\beta)^{-1} + 1 \bigr).
    \end{equation*}
    For the second sum, recall by~\cref{eq:loss-decomposition} that
    \begin{equation*}
        \begin{aligned}
            \sum_{t=1}^T \gamma^{T-t} \bigl( \mu^{k_t^\star}(\theta_t) &- \mu^{k_t}(\theta_t) \bigr)
            = \sum_{t=1}^T \gamma^{T-t} \bigl( \tilde\mu_{t-1}^{k_t^\star} - \tilde\mu_{t-1}^{k_t} \bigr) \\
            &~ + \sum_{t=1}^T \gamma^{T-t} \bigl( \mu^{k_t^\star}(\theta_t) - \tilde\mu_{t-1}^{k_t^\star} \bigr) - \sum_{t=1}^T \gamma^{T-t} \bigl( \mu^{k_t}(\theta_t) - \tilde\mu_{t-1}^{k_t} \bigr)
        \end{aligned}.
    \end{equation*}
    By~\cref{th:bandit-gap-ducb}, the first term on the RHS satisfies
    \begin{equation*}
        \begin{aligned}
            \sum_{t=K+1}^T \mathbbm{1} \left\{ N_{t-1}^{k_t}(\beta) \ge A \right\} \cdot \gamma^{T-t} \Bigl( \tilde\mu_{t-1}^{k_t^\star} - \tilde\mu_{t-1}^{k_t} \Bigr)
            &\le 2 \sqrt{1 + A^{-1}} \sum_{t=K+1}^T \gamma^{T-t} r_t^{k_t} \\
            &\le 4 \sum_{t=K+1}^T \gamma^{T-t} r_t^{k_t} \\
        \end{aligned},
    \end{equation*}
    with probability \textit{at least} $1-2H(\beta)^{-\alpha}$, where $A^{-1}\le3$ by construction.
    Thus it follows from a union bound that
    \begin{equation*}
        \begin{aligned}
            \Pr &\left[ \sum_{t=K+1}^T \mathbbm{1} \left\{ N_{t-1}^{k_t}(\beta) \ge A \right\} \cdot \gamma^{T-t} \Bigl( \tilde\mu_{t-1}^{k_t^\star} - \tilde\mu_{t-1}^{k_t} \Bigr) \ge 4 \sum_{t=K+1}^T \gamma^{T-t} r_t^{k_t} \right] \\
            &\le \sum_{t=K+1}^T \Pr \left[ \tilde\mu_{t-1}^{k_t^\star} - \tilde\mu_{t-1}^{k_t} \ge 4 r_t^{k_t}, N_{t-1}^{k_t}(\beta) \ge A \right] \\
            &\le 2 T H(\beta)^{-\alpha} \\
        \end{aligned}.
    \end{equation*}
    By the Cauchy-Schwarz inequality, we have $\sum_{k=1}^K\sqrt{N_T^k(\beta)}\le\sqrt{KN_T(\beta)}$.
    Thus
    \begin{equation*}
        \begin{aligned}
            4 \sum_{t=K+1}^T \gamma^{T-t} r_t^{k_t}
            &= 8 C \sqrt{2 \alpha \ln H(\beta)} \sum_{t=K+1}^T \frac{\gamma^{T-t}}{\sqrt{N_t^k(\beta)}} \\
            &= 8 C \sqrt{2 \alpha \ln H(\beta)} \sum_{k=1}^K \sum_{t=K+1}^T \frac{\mathbbm{1} \left\{ k_t = k \right\} \cdot \gamma^{T-t}}{\sqrt{N_t^k(\beta)}} \\
            &\le 8 C \sqrt{2 \alpha \ln H(\beta)} \sum_{k=1}^K \Bigl( 2 N_T^k(\beta)^\frac{1}{2} + N_T^k(\gamma) H(\beta)^{-\frac{1}{2}} \Bigr) \\
            &\le 8 C \sqrt{2 \alpha \ln H(\beta)} \bigl( 2 K^\frac{1}{2} N_T(\beta)^\frac{1}{2} + N_T(\gamma) H(\beta)^{-\frac{1}{2}} \bigr) \\
        \end{aligned},
    \end{equation*}
    where the first inequality follows from~\cref{th:discount-sum}.
    Further combining with~\cref{th:ema-bias}, we have
    \begin{equation*}
        \begin{aligned}
            \sum_{t=K+1}^T \mathbbm{1} &\left\{ N_{t-1}^{k_t}(\beta) \ge A \right\} \cdot \gamma^{T-t} \Bigl( \tilde\mu_{t-1}^{k_t^\star} - \tilde\mu_{t-1}^{k_t} \Bigr) \\
            \le&~ 8 C \sqrt{2 \alpha \ln H(\beta)} \bigl( 2 K^\frac{1}{2} N_T(\beta)^\frac{1}{2} + N_T(\gamma) H(\beta)^{-\frac{1}{2}} \bigr) \\
            &~ + 2 \eta G^2 N_T(\gamma) H(\beta) \ln \left( \frac{2 C e}{\eta A G^2} \right) \\
        \end{aligned},
    \end{equation*}
    with probability \textit{at least} $2TH(\beta)^{-\alpha}$.
    And when the bound fails, the additional regret incurred is \textit{at most} $2CTN_T(\gamma)H(\beta)^{-\alpha}$.
    Putting together,
    \begin{equation*}
        \begin{aligned}
            \mathbb{E} \bigl[ R_b^\gamma(\operatorname{D-UCB}_T) \bigr]
            \le&~ \gamma^{T-K} C K + 4 A C K \bigl( T N_T(\beta)^{-1} + 1 \bigr) \\
            &~ + 8 C \sqrt{2 \alpha \ln H(\beta)} \bigl( 2 K^\frac{1}{2} N_T(\beta)^\frac{1}{2} + N_T(\gamma) H(\beta)^{-\frac{1}{2}} \bigr) \\
            &~ + 2 \eta G^2 N_T(\gamma) H(\beta) \ln \left( \frac{2 C e}{\eta A G^2} \right) \\
            &~ + 2 C N_T(\gamma) H(\beta)^{1-\alpha} \\
        \end{aligned}
    \end{equation*}
    Setting a sufficiently large $\alpha$ (\eg $\alpha=(\ln T+\ln H(\gamma))/\ln H(\beta)$ \st $TN_T(\gamma)H(\beta)^{1-\alpha}\le1$) concludes the proof.
\end{proof}

\begin{proposition}[Normalized Dynamic Bandit Regret of DUCB]
    The normalized dynamic bandit regret of DUCB satisfies
    \begin{equation*}
        \frac{T}{N_T(\gamma)} \mathbb{E} \bigl[ R_b^\gamma(\operatorname{DUCB}_T) \bigr] \le
        \begin{cases}
            \begin{aligned}
                & \tilde{O} \bigl( K^\frac{1}{2} T^\frac{1}{2} \bigr) & \eta \le O(T^{-\frac{2}{3}}) \\
                & \tilde{O} \bigl( \eta^\frac{1}{3} K^\frac{1}{3} T \bigr) & O(T^{-\frac{3}{2}})\le\eta\le o(1) \\
            \end{aligned}
        \end{cases}
    \end{equation*}
    for any $\gamma\in[1-T^{-1},1]$.
\end{proposition}

\begin{proof}
    First write
    \begin{equation*}
        \begin{aligned}
            \frac{T}{N_T(\gamma)} \mathbb{E} \bigl[ R_b^\gamma(\operatorname{DUCB}_T) \bigr]
            \le&~ O \bigl( K T^2 N_T(\gamma)^{-1} H(\beta)^{-1} \bigr) + \tilde{O} \bigl( K^\frac{1}{2} T N_T(\gamma)^{-1} H(\beta)^\frac{1}{2} \bigr) \\
            &~ + \tilde{O} \bigl( T H(\beta)^{-\frac{1}{2}} \bigr) + \tilde{O} \bigl( \eta T H(\beta) \bigr) \\
        \end{aligned}.
    \end{equation*}
    Note that when $\gamma\ge1-T^{-1}$, we have
    \begin{equation*}
        \frac{T}{N_T(\gamma)} \le \frac{1}{1 - (1 - T^{-1})^T} \le \frac{1}{1 - e^{-1}} \le 2,
    \end{equation*}
    since $(1-T^{-1})^T\nearrow e^{-1}$ as $T\to\infty$.
    Then it follows that
    \begin{equation*}
        \begin{aligned}
            \frac{T}{N_T(\gamma)} \mathbb{E} \bigl[ R_b^\gamma(\operatorname{DUCB}_T) \bigr]
            \le&~ O \bigl( K T H(\beta)^{-1} \bigr) + \tilde{O} \bigl( K^\frac{1}{2} H(\beta)^\frac{1}{2} \bigr) \\
            &~ + \tilde{O} \bigl( T H(\beta)^{-\frac{1}{2}} \bigr) + \tilde{O} \bigl( \eta T H(\beta) \bigr) \\
        \end{aligned}.
    \end{equation*}
    Now we discuss the choice of $\beta$ over different values of $\eta$.
    When $\eta\le o(T^{-{3/2}})$, \ie the environment is \textit{almost stationary}, the last term is negligible since $O(\eta TH(\beta))\le o(T^{1/2})$. Thus we can choose $\beta=1-T^{-1}$ \st $H(\beta)=T$, and
    \begin{equation*}
        \begin{aligned}
            \frac{T}{N_T(\gamma)} \mathbb{E} \bigl[ R_b^\gamma(\operatorname{DUCB}_T) \bigr]
            &\le O \bigl( K \bigr) + \tilde{O} \bigl( K^\frac{1}{2} T^\frac{1}{2} \bigr) + \tilde{O} \bigl( T^\frac{1}{2} \bigr) + o \bigl( T^\frac{1}{2} \bigr) \\
            &= \tilde{O} \bigl( K^\frac{1}{2} T^\frac{1}{2} \bigr) \\
        \end{aligned}.
    \end{equation*}
    When $O(T^{-3/2})\le\eta\le o(1)$, \ie the environment is \textit{slowly changing}, we have $\eta K\le1$ for any fixed $K$ independent of $T$. Let $\beta=1-\eta^{2/3}K^{-1/3}$ \st $H(\beta)=\eta^{-2/3}K^{1/3}$. We have
    \begin{equation*}
        \begin{aligned}
            \frac{T}{N_T(\gamma)} \mathbb{E} \bigl[ R_b^\gamma(\operatorname{DUCB}_T) \bigr]
            &\le O \bigl( \eta^\frac{2}{3} K^\frac{2}{3} T \bigr) + \tilde{O} \bigl( \eta^\frac{1}{3} K^\frac{1}{3} \bigr) + \tilde{O} \bigl( \eta^\frac{1}{3} K^{-\frac{1}{6}} T \bigr) + \tilde{O} \bigl( \eta^\frac{1}{3} K^\frac{1}{3} T \bigr) \\
            &= \tilde{O} \bigl( \eta^\frac{1}{3} K^\frac{1}{3} T \bigr) \\
        \end{aligned}.
    \end{equation*}
    Note that if $\eta=O(1)$, \ie the environment is \textit{abruptly changing}, no algorithm can achieve a sublinear regret given our setting where $\theta_t$ changes in all steps~\cite{kivinen_upper-confidence_2011}.
\end{proof}

This proves~\cref{th:bandit-regret}.

\subsection{Proof of \texorpdfstring{\cref{th:bandit-gap-ducb}}{Bandit Gap of DUCB}}
\label{sec:proofs-gap}

First we state a discounted version of Azuma's inequality.
Consider the sequence
\begin{equation*}
    \left\{ Z_t^k := \mathbbm{1} \left\{ k_t = k \right\} \cdot \left( \hat\mu^k(\theta_t;\mathcal{B}_t) - \mu^k(\theta_t) \right), t = 1,\dots,T \right\}, \quad k \in [K].
\end{equation*}
Note that $\theta_t$ is $\sigma(\mathcal{A}_{t-1},\mathcal{X}_{t-1})$-measurable by the design of OGD. Further since
\begin{equation*}
    \begin{aligned}
        \mathbb{E} \left[ Z_t^k \middle| \mathcal{A}_{t-1}, \mathcal{X}_{t-1} \right]
        &= \mathbb{E} \left[ \mathbbm{1} \left\{ k_t = k \right\} \cdot \left( \hat\mu^k(\theta_t;\mathcal{B}_t) - \mu^k(\theta_t) \right) \middle| \theta_t \right] \\
        &= \Pr[k_t = k] \cdot \left( \mathbb{E} \left[ \hat\mu^k(\theta_t;\mathcal{B}_t) \middle| k_t, \theta_t \right] - \mu^k(\theta_t) \right) + \Pr[k_t \ne k] \cdot 0 \\
        &= 0 \\
    \end{aligned},
\end{equation*}
the sequence $\{Z_t^k,t=1,\dots,T\}$ is a \textit{martingale difference sequence} for all fixed $k\in[K]$.
We have the following lemma.

\begin{proposition}[Discounted Azuma's Inequality]
    For any constant $\alpha>0$ and $\beta\in[0,1]$,
    \begin{equation*}
        \Pr \left[ \sum_{\tau=1}^t \beta^{t-\tau} Z_\tau^k > 2 C \sqrt{(\alpha + 1) N_t^k(\beta^2) \ln H(\beta)} \right] \le H(\beta)^{-\alpha},
    \end{equation*}
    where $N_t^k(\cdot)$ is defined in~\cref{eq:ema} and $H(\cdot)$ is defined in~\cref{eq:ducb}.
    \label{th:azuma}
\end{proposition}

\begin{proof}
    The proof is an adaptation of the standard proof of Azuma's inequality, where we specifically deal with the randomness of $N_t^k(\beta^2)$.
    Following~\citet{kivinen_upper-confidence_2011}, we apply a union bound over its possible values using the \textit{peeling technique}.
    Specifically, let
    \begin{equation*}
        M(\beta,\delta) := \left\lceil \frac{\ln H(\beta)}{\ln(1 + \delta)} \right\rceil, \quad \delta > 0.
    \end{equation*}
    Note that since $1-\beta\le\ln(1/\beta)\le1-1/\beta$, we have
    \begin{equation*}
        \frac{\beta}{1 - \beta} \le H(\beta) \le \frac{1}{1 - \beta}.
    \end{equation*}
    Define the intervals $I_m:=(n_{m-1},n_m]$ for each $m\in[M(\beta,\delta)]$, where we set $n_0:=0$ and $n_m:=(1+\delta)^m$. Then
    \begin{equation*}
        n_M = (1 + \delta)^{M(\beta,\delta)} \ge H(\beta) = \frac{1}{1 - \beta} \ge \frac{1}{1 - \beta^2} \ge N_t^k(\beta^2),
    \end{equation*}
    and there exists a unique $m\in[M(\beta,\delta)]$ \st $N_t^k(\beta^2)\in I_m$.
    Then for any $\epsilon>0$,
    \begin{equation*}
        \Pr \left[ \sum_{\tau=1}^t \beta^{t-\tau} Z_\tau^k > \epsilon \right] = \sum_{m=1}^{M(\beta,\delta)} \Pr \left[ \sum_{\tau=1}^t \beta^{t-\tau} Z_\tau^k > \epsilon, N_t^k(\beta^2) \in I_m \right].
    \end{equation*}
    For each event where $N_t^k(\beta)\in I_m$, there is $N_t^k(\beta^2)\le n_m$, and thus
    \begin{equation*}
        \Pr \left[ \sum_{\tau=1}^t \beta^{t-\tau} Z_\tau^k > \epsilon, N_t^k(\beta^2) \in I_m \right] \le \Pr \left[ \sum_{\tau=1}^t \beta^{t-\tau} Z_\tau^k - \frac{\lambda C^2}{2} N_t^k(\beta^2) >\epsilon - \frac{\lambda C^2}{2} n_m \right].
    \end{equation*}
    Then applying Chernoff's bound to the RHS gives
    \begin{equation*}
        \begin{aligned}
            \Pr &\left[ \sum_{\tau=1}^t \gamma^{t-\tau} Z_\tau^k > \epsilon, N_t^k(\beta^2) \in I_m  \right] \\
            &\le \inf_{\lambda>0} ~ \exp \left( -\lambda \epsilon + \frac{\lambda^2 C^2}{2} n_m \right) \mathbb{E} \left[ \exp \left( \lambda \sum_{\tau=1}^t \gamma^{t-\tau} Z_\tau^k - \frac{\lambda^2 C^2}{2} N_t^k(\beta^2) \right) \right] \\
        \end{aligned}.
    \end{equation*}
    Recall by Hoeffding's lemma that
    \begin{equation*}
        \mathbb{E} \left[ \exp(\lambda \beta^{t-\tau} Z_\tau^k) | \mathcal{A}_\tau, \mathcal{X}_{\tau-1} \right] \le \exp \left( \mathbbm{1} \left\{ k_\tau = k \right\} \cdot \frac{\lambda^2 \beta^{2(t-\tau)} C^2}{2} \right).
    \end{equation*}
    Then
    \begin{equation*}
        \begin{aligned}
            \mathbb{E} &\left[ \exp \left( \lambda \sum_{\tau=1}^t \beta^{t-\tau} Z_\tau^k - \frac{\lambda^2 C^2}{2} N_t^k(\beta^2) \right) \right] \\
            &= \mathbb{E} \left[ \exp \left( \lambda \sum_{\tau=1}^{t-1} \beta^{t-\tau} Z_\tau^k - \frac{\lambda^2 C^2}{2} N_t^k(\beta^2) \right) \mathbb{E} \left[ \exp(\lambda Z_t^k) \middle| \mathcal{A}_t, \mathcal{X}_{t-1} \right] \right] \\
            &\le \mathbb{E} \left[ \exp \left( \lambda \sum_{\tau=1}^{t-1} \beta^{t-\tau} Z_\tau^k - \frac{\lambda^2 C^2}{2} N_t^k(\beta^2) \right) \cdot \exp \left( \mathbbm{1} \left\{ k_t = k \right\} \cdot \frac{\lambda^2 C^2}{2} \right) \right] \\
            &\le \dots \\
            &\le \mathbb{E} \left[ \exp \left( -\frac{\lambda^2 C^2}{2} N_t^k(\beta^2) + \frac{\lambda^2 C^2}{2} \sum_{\tau=1}^t \mathbbm{1} \left\{ k_\tau = k \right\} \cdot \beta^{2(t-\tau)} \right) \right] \\
            &= 1 \\
        \end{aligned}.
    \end{equation*}
    Therefore,
    \begin{equation*}
        \Pr \left[ \sum_{\tau=1}^t \beta^{t-\tau} Z_\tau^k > \epsilon, N_t^k(\beta^2) \in I_m  \right] \le \inf_{\lambda>0} ~ \exp \left( -\lambda \epsilon + \frac{\lambda^2 C^2}{2} n_m \right) = \exp \left( -\frac{\epsilon^2}{2 C^2 n_m} \right),
    \end{equation*}
    where the RHS of the inequality is minimized by taking $\lambda=\epsilon/(C^2n_m)$.
    Since $n_m\le(1+\delta)N_t^k(\beta^2)$ for any $N_t^k(\beta^2)\in I_m$, we have
    \begin{equation*}
        \begin{aligned}
            \Pr &\left[ \sum_{\tau=1}^t \beta^{t-\tau} Z_\tau^k > C \sqrt{2 (1 + \delta) N_t^k(\beta^2) \ln\frac{M(\beta,\delta)}{H(\beta)^{-\alpha}}}, N_t^k(\beta^2) \in I_m \right] \\
            &\le \Pr \left[ \sum_{\tau=1}^t \beta^{t-\tau} Z_\tau^k > C \sqrt{2 n_m \ln\frac{M(\beta,\delta)}{H(\beta)^{-\alpha}}}, N_t^k(\beta^2) \in I_m \right] \\
            &\le \frac{H(\beta)^{-\alpha}}{M(\beta,\delta)} \\
        \end{aligned}.
    \end{equation*}
    Combining together, we have
    \begin{equation*}
        \Pr \left[ \sum_{\tau=1}^t \beta^{t-\tau} Z_\tau^k > C \sqrt{2 (1 + \delta) N_t^k(\beta^2) \ln\frac{M(\beta,\delta)}{H(\beta)^{-\alpha}}} \right] \le H(\beta)^{-\alpha}.
    \end{equation*}
    Finally, we choose $\delta>0$ \st $M(\beta,\delta)=H(\beta)$, \ie as $\beta\to1$,
    \begin{equation*}
        1 + \delta = H(\beta)^{1/H(\beta)} \nearrow e^\frac{1}{e} < 2.
    \end{equation*}
    This concludes the proof.
\end{proof}

Then we immediately have the following result.

\begin{lemma}
    For any $\alpha\ge1$,
    \begin{equation*}
        \Pr \left[ \hat\mu_t^k - \tilde\mu_t^k > 2 C \sqrt{\frac{2 \alpha \ln H(\beta)}{N_t^k(\beta)}} \right] \le H(\beta)^{-\alpha}.
    \end{equation*}
    \label{th:loss-concentration}
\end{lemma}

\begin{proof}
    Note that
    \begin{equation*}
        \hat\mu_t^k - \tilde\mu_t^k = \frac{1}{N_t^k(\gamma)} \sum_{\tau=1}^t \gamma^{t-\tau} Z_\tau^k.
    \end{equation*}
    Since $N_t^k(\gamma)\ge N_t^k(\gamma^2)$ for all $\gamma\in(0,1)$, we have
    \begin{equation*}
        \frac{C}{N_t^k(\gamma)} \sqrt{(\alpha + 1) N_t^k(\gamma^2) \ln H(\gamma)} \le C \sqrt{\frac{(\alpha + 1)  \ln H(\gamma)}{N_t^k(\gamma)}} \le C \sqrt{\frac{2 \alpha \ln H(\gamma)}{N_t^k(\gamma)}},
    \end{equation*}
    for all $\alpha\ge1$.
    Then
    \begin{equation*}
        \begin{aligned}
            \Pr \left[ \hat\mu_t^k - \tilde\mu_t^k > 2 C \sqrt{\frac{2 \alpha \ln H(\gamma)}{N_t^k(\gamma)}} \right]
            &\le \Pr \left[ \sum_{\tau=1}^t \gamma^{t-\tau} Z_\tau^k > 2 C \sqrt{(\alpha + 1) N_t^k(\gamma)  \ln H(\gamma)} \right] \\
            &\le H(\gamma)^{-\alpha} \\
        \end{aligned},
    \end{equation*}
    where the last inequality follows from~\cref{th:azuma}.
\end{proof}

Now we can prove the DUCB bandit gap.

\begin{proof}[Proof of~\cref{th:bandit-gap-ducb}]
    By the rule of DUCB,
    \begin{equation*}
        \hat\mu_{t-1}^k + r_{t-1}^k \le \hat\mu_{t-1}^{k_t} + r_{t-1}^{k_t}, \quad \forall k \in [K].
    \end{equation*}
    Specifically for $k=k_t^\star$, there is
    \begin{equation*}
        \hat\Delta_{t-1}^{k_t} := \hat\mu_{t-1}^{k_t^\star} - \hat\mu_{t-1}^{k_t} \le r_{t-1}^{k_t} - r_{t-1}^{k_t^\star}.
    \end{equation*}
    When $N_{t-1}^k(\beta)\ge A$, we have
    \begin{equation*}
        \frac{1 + A^{-1}}{N_t(\beta)} = \frac{N_{t-1}^k(\beta) + A^{-1} N_{t-1}^k(\beta)}{N_{t-1}^k(\beta) N_t(\beta)} \ge \frac{N_{t-1}^k(\beta) + 1}{N_{t-1}^k(\beta) N_t(\beta)} \ge \frac{1}{N_{t-1}^k(\beta)},
    \end{equation*}
    where the last inequality is due to $N_{t-1}^k(\beta)+1\ge N_t^k(\beta)$.
    Thus
    \begin{equation*}
        r_t^k \sqrt{1 + A^{-1}} = 2 C \sqrt{\frac{1 + A^{-1}}{N_t^k(\beta)} \cdot 2 \alpha \ln H(\beta)} \ge 2 C \sqrt{\frac{2 \alpha \ln H(\beta)}{N_{t-1}^k(\beta)}} = r_{t-1}^k.
    \end{equation*}
    Then by the standard UCB argument,
    \begin{equation*}
        \begin{aligned}
            \Pr &\left[ \tilde\mu_{t-1}^{k_t^\star} - \tilde\mu_{t-1}^{k_t} > 2 r_t^{k_t} \sqrt{1 + A^{-1}}, N_{t-1}^{k_t}(\beta) \ge A \right] \\
            &\le \Pr \left[ \tilde\mu_{t-1}^{k_t^\star} - \tilde\mu_{t-1}^{k_t} > 2 r_{t-1}^{k_t} \right] \\
            &= \Pr \left[ \tilde\mu_{t-1}^{k_t^\star} - \tilde\mu_{t-1}^{k_t} - \hat\Delta_{t-1}^{k_t} > 2 r_{t-1}^{k_t} - \hat\Delta_{t-1}^{k_t} \right] \\
            &\le \Pr \left[ \bigl( \tilde\mu_{t-1}^{k_t^\star} - \hat\mu_{t-1}^{k_t^\star} \bigr) - \bigl( \tilde\mu_{t-1}^{k_t} - \hat\mu_{t-1}^{k_t} \bigr) > r_{t-1}^{k_t^\star} + r_{t-1}^{k_t} \right] \\
            &\le \Pr \Bigl[ \hat\mu_{t-1}^{k_t^\star} - \tilde\mu_{t-1}^{k_t^\star} < -r_{t-1}^{k_t^\star} \Bigr] + \Pr \Bigl[ \hat\mu_{t-1}^{k_t} - \tilde\mu_{t-1}^{k_t} > r_{t-1}^{k_t} \Bigr] \\
        \end{aligned}.
    \end{equation*}
    Applying~\cref{th:loss-concentration} respectively on each term concludes the proof.
\end{proof}

\subsection{Proof of \texorpdfstring{\cref{th:ema-bias}}{EMA Bias}}
\label{sec:proofs-bias}

We prove this by splitting the EMA summands by a horizon $D(\beta)$ at each step $t$.
Note that the early summands for $\tau=1,\dots,\lfloor{D(\beta)}\rfloor$ inherits more bias, since they are calculated based on some model parameter drastically different from the current parameter $\theta_t$. Thus we simply apply the boundedness condition.
On the other hand, more recent summands for $\tau=\lceil{D(\beta)}\rceil,\dots,t$ are calculated based on the model parameters much closer to the current $\theta_t$. Thus we bound them using the Lipschitz condition in~\cref{as:loss}.
Finally, we optimize the choice of $D(\beta)$ to obtain~\cref{th:ema-bias}.

\begin{proof}[Proof of~\cref{th:ema-bias}]
    Suppose there exists some $0\le D(\beta)\le t$ for $t\in[T]$.
    We can write
    \begin{equation*}
        \begin{aligned}
            \left| \tilde\mu_t^k - \mu^k(\theta_{t+1}) \right|
            =&~ \frac{1}{N_t^k(\beta)} \left| \sum_{\tau=1}^t \mathbbm{1} \left\{ k_\tau = k \right\} \cdot \beta^{t-\tau} \left( \mu^k(\theta_\tau) - \mu^k(\theta_{t+1}) \right) \right| \\
            \le&~ \frac{1}{N_t^k(\beta)} \left| \sum_{\tau=1}^{t-\lfloor D(\beta)\rfloor} \mathbbm{1} \left\{ k_\tau = k \right\} \cdot \beta^{t-\tau} \left( \mu^k(\theta_\tau) - \mu^k(\theta_{t+1}) \right) \right| \\
            &~ + \frac{1}{N_t^k(\beta)} \left| \sum_{\tau=t-\lceil D(\beta)\rceil}^t \mathbbm{1} \left\{ k_\tau = k \right\} \cdot \beta^{t-\tau} \left( \mu^k(\theta_\tau) - \mu^k(\theta_{t+1}) \right) \right| \\
        \end{aligned}.
    \end{equation*}
    For the first sum, since $|\mu^k(\theta_\tau)-\mu^k(\theta_t)|<C$ for all $\tau,t\in[T]$, we have
    \begin{equation*}
        \begin{aligned}
            \frac{1}{N_t^k(\beta)} &\left| \sum_{\tau=1}^{t-\lfloor D(\beta)\rfloor} \mathbbm{1} \left\{ k_t = k \right\} \cdot \beta^{t-\tau} \left( \mu^k(\theta_\tau) - \mu^k(\theta_{t+1}) \right) \right| \\
            &\le \frac{C \beta^{\lfloor D(\beta)\rfloor}}{N_t^k(\beta)} \sum_{\tau=1}^{t-\lfloor D(\beta)\rfloor} \mathbbm{1} \left\{ k_\tau = k \right\} \cdot \beta^{t-\lfloor D(\beta)\rfloor-\tau} \\
            &= \frac{C \beta^{D(\beta)} N_{t-\lfloor D(\beta)\rfloor}^k(\beta)}{N_t^k(\beta)} \\
            &\le C A^{-1} H(\beta) \beta^{D(\beta)} \\
        \end{aligned}.
    \end{equation*}
    For the second sum, by the Lipschitz condition and the rule of OGD, we have
    \begin{equation*}
        \begin{aligned}
            \max_{\tau\in[t-\lceil D(\beta)\rceil,t]} ~ \left| \mu^k(\theta_\tau) - \mu^k(\theta_{t+1}) \right|
            &\le G \left\| \sum_{\tau=t-\lceil D(\beta)\rceil}^t \eta \nabla\hat\mu^{k_t}(\theta_\tau;\cdot) \right\| \\
            &\le \eta G^2 D(\beta) \\
        \end{aligned}.
    \end{equation*}
    Then
    \begin{equation*}
        \left| \tilde\mu_t^k - \mu^k(\theta_{t+1}) \right| \le C A^{-1} H(\beta) \beta^{D(\beta)} + \eta G^2 D(\beta).
    \end{equation*}
    To choose the optimal $D(\beta)$, let
    \begin{equation*}
        \frac{d}{d D(\beta)} \left| \tilde\mu_t^k - \mu^k(\theta_{t+1}) \right| = C A^{-1} H(\beta) (\ln\beta) \beta^{D(\beta)} + \eta G^2 = 0,
    \end{equation*}
    \ie
    \begin{equation*}
        D^\star(\beta) = \frac{1}{\ln \beta} \ln \left( \frac{\eta A G^2}{C H(\beta) \ln(1/\beta)} \right).
    \end{equation*}
    Further note that the LHS is \textit{convex} in $D(\beta)$. Thus the assumption that $0\le D(\beta)\le t$ can be alleviated, and
    \begin{equation*}
        \left| \tilde\mu_t^k - \mu^k(\theta_{t+1}) \right| \le \frac{\eta G^2}{\ln(1/\beta)} \left( 1 + \ln \left( \frac{C H(\beta) \ln(1/\beta)}{\eta A G^2} \right) \right).
    \end{equation*}
    Finally, since $\ln(1/\beta)\ge1-\beta=H(\beta)^{-1}$, we conclude the proof by realizing that the RHS above is monotonically decreasing \wrt $\ln(1/\beta)$, and $H(\beta)\ln(1/\beta)\le1/\beta\le2$ for all $\beta\ge1/2$.
\end{proof}

\subsection{Analysis of Sampling DUCB}
\label{sec:proofs-sducb}

To show that the bandit regret of the sampling version of DUCB is a constant factor away from the standard DUCB, we only need to show that the bandit gap is a constant factor away from~\cref{th:bandit-gap-ducb}, and the rest naturally follows.

\begin{proposition}[Bandit Gap of Sampling DUCB]
    Suppose sampling DUCB picks $k_t$ in step $t\ge K+1$ as in~\cref{sec:impl}. Then for any $A>0$,
    \begin{equation*}
        \Pr \left[ \tilde\mu_{t-1}^{k_t^\star} - \tilde\mu_{t-1}^{k_t} > 4 r_t^{k_t} \sqrt{1 + A^{-1}}, N_{t-1}^{k_t}(\beta) \ge A \right] \le 4 H(\beta)^{-\alpha}.
    \end{equation*}
\end{proposition}

\begin{proof}
    By the rule of sampling DUCB, we have
    \begin{equation*}
        \hat\mu_{t-1}^k + \phi_{t-1}^k \le \hat\mu_{t-1}^{k_t} + \phi_{t-1}^{k_t}, \quad \forall k \in [K].
    \end{equation*}
    And specifically for $k=k_t^\star$ (recall that $k_t^\star=\argmax_k\mu^k(\theta_t)$),
    \begin{equation*}
        \hat\Delta_t^{k_t} = \hat\mu_{t-1}^{k_t^\star} - \hat\mu_{t-1}^{k_t} \le \phi_{t-1}^{k_t} - \phi_{t-1}^{k_t^\star}.
    \end{equation*}
    According to the Gaussian tail bound,
    \begin{equation*}
        \Pr \left[ \phi_{t-1}^k - 2 r_{t-1}^k > r_{t-1}^k \right] \le H(\beta)^{-\alpha}, \quad
        \Pr \left[ \phi_{t-1}^k - 2 r_{t-1}^k < -r_{t-1}^k \right] \le H(\beta)^{-\alpha}.
    \end{equation*}
    Recall by~\cref{th:bandit-gap-ducb} that $(1+A^{-1})N_t^k(\beta)^{-1}\ge N_{t-1}^k(\beta)^{-1}$. Then
    \begin{equation*}
        \begin{aligned}
            \Pr &\left[ \tilde\mu_{t-1}^{k_t^\star} - \tilde\mu_{t-1}^{k_t} > 4 r_t^{k_t} \sqrt{1 + A^{-1}}, N_{t-1}^{k_t}(\beta) \ge A \right] \\
            &\le \Pr \left[ \tilde\mu_{t-1}^{k_t^\star} - \tilde\mu_{t-1}^{k_t} > 4 r_{t-1}^{k_t} \right] \\
            &= \Pr \left[ \tilde\mu_{t-1}^{k_t^\star} - \tilde\mu_{t-1}^{k_t} - \hat\Delta_t^{k_t} + \hat\Delta_t^{k_t} > 4 r_{t-1}^{k_t} \right] \\
            &\le \Pr \left[ \bigl( \tilde\mu_{t-1}^{k_t^\star} - \hat\mu_{t-1}^{k_t^\star} \bigr) - \bigl( \tilde\mu_{t-1}^{k_t} - \hat\mu_{t-1}^{k_t} \bigr) + \phi_{t-1}^{k_t} - \phi_{t-1}^{k_t^\star} > 4 r_{t-1}^{k_t} + r_{t-1}^{k_t^\star} -r_{t-1}^{k_t^\star} \right] \\
            &\le \Pr \Bigl[ \hat\mu_{t-1}^{k_t^\star} - \tilde\mu_{t-1}^{k_t^\star} < -r_{t-1}^{k_t^\star} \Bigr] + \Pr \Bigl[ \hat\mu_{t-1}^{k_t} - \tilde\mu_{t-1}^{k_t}(\theta_t) > r_{t-1}^{k_t} \Bigr] \\
            &~~~~ + \Pr \Bigl[ \phi_{t-1}^{k_t} > 3 r_{t-1}^{k_t} \Bigr] + \Pr \Bigl[ \phi_{t-1}^{k_t^\star} < r_{t-1}^{k_t^\star} \Bigr] \\
        \end{aligned},
    \end{equation*}
    where the last inequality follows from a union bound.
    Applying \cref{th:loss-concentration} and the Gaussian tail bound respectively on each term concludes the proof.
\end{proof}

\newpage

\section{Additional Experimental Results}
\label{sec:results}

This section provides additional results beyond the compact~\cref{tab:results-acc}. The goal is two-fold. First, we evaluate whether the behavior of DUCB transfers across model families and parameter scales. Second, we report auxiliary metrics, including mean token accuracy for SFT and reward margin for DPO, to verify that improvements in the primary robustness metric do not arise from degenerate behavior on secondary measures.

All models are trained either on a single 80G or two 40G NVIDIA A100 GPU(s), using the AdamW optimizer with 8-bit LoRA quantization and bfloat-16 mixed precision.
The training was conducted using the NSF COSMOS Testbed~\citep{raychaudhuri_cosmos_2020} and the Empire AI~\citep{bloom_empire_2025} resources.
For each set of experiments, we set a uniform \textit{sample budget} and explicitly avoid training for multiple epochs. The goal is to ensure that every training mini-batch is non-repetitive and aligns with~\cref{as:iid}.
Detailed configurations and scripts will be released with the code upon publication.

\subsection{SFT}
\label{sec:results-sft}

\begin{table}[ht]
  \caption{SFT results on the GSM8k + IFEval datasets with a budget of 32k samples.}
  \label{tab:results-sft}
  \centering
  \begin{tabular}{ccccccc}
    \toprule
    \multirow{2}{*}[-0.5ex]{\textbf{Model}} & \multirow{2}{*}[-0.5ex]{\textbf{Sampler}} & \multirow{2}{*}[-0.5ex]{\textbf{Signal}} & \multicolumn{2}{c}{\makecell{\textbf{Accuracy}}} & \multicolumn{2}{c}{\makecell{\textbf{Mean Token}\\\textbf{Accuracy}}} \\
    \cmidrule(lr){4-5} \cmidrule(lr){6-7} & & & \textbf{\textit{Min}} & \textbf{Avg} & \textbf{Min} & \textbf{Avg} \\
    \midrule
    \multirow{4}{*}{\makecell{Qwen2.5-3B\\(Base)}} &
        Uniform & Loss & 22.85 & \textbf{29.79} & 81.23 & 84.97 \\ &
        Exp3 & Loss & 23.63 & 25.68 & 78.82 & 82.35 \\ &
        Exp3 & Loss Difference$^1$ & 23.63 & 27.24 & 81.46 & 84.78 \\ &
        DUCB (Ours) & Loss & \textbf{24.02} & 26.95 & 81.35 & 84.71 \\
    \midrule
    \multirow{4}{*}{\makecell{Qwen2.5-7B\\(Base)}} &
        Uniform & Loss & 27.15 & 39.36 & 78.79 & 82.34 \\ &
        Exp3 & Loss & 27.54 & 37.60 & 81.42 & 84.76 \\ &
        Exp3 & Loss Difference$^1$ & 27.73 & \textbf{40.72} & 78.87 & 82.35 \\ &
        DUCB (Ours) & Loss & \textbf{28.52} & 37.89 & 78.82 & 82.35 \\
    \midrule
    \multirow{4}{*}{\makecell{Llama3.1-8B\\(Base)}} &
        Uniform & Loss & 22.07 & 25.98 & 70.31 & 73.66 \\ &
        Exp3 & Loss & 23.24 & 26.27 & 71.27 & 73.63 \\ &
        Exp3 & Loss Difference$^1$ & 21.48 & \textbf{27.93} & 71.08 & 73.57 \\ &
        DUCB (Ours) & Loss & \textbf{24.22} & 25.29 & 71.22 & 73.61 \\
    \bottomrule
    \addlinespace[2pt]
    \multicolumn{7}{@{}l}{\footnotesize $^1$Follows from~\citet{shin_dynamixsft_2025}.}
  \end{tabular}
\end{table}

We evaluate the SFT performance on a heterogeneous dataset comprising GSM8k~\cite{cobbe_training_2021} and IFEval~\cite{zhou_instruction-following_2023}, evaluating the model's ability to balance mathematical reasoning with instruction-following.
We compare against Exp3 utilizing standard SFT loss and the DynamixSFT-style \textit{stepwise loss difference} proposed by~\citet{shin_dynamixsft_2025}. We approximate their forward \textit{anticipated} loss probing using virtual parameter update (followed parameter reset) with backward \textit{observed} loss subtraction.

The observations confirm that DUCB consistently fulfills its theoretical promise of robustness, achieving the highest minimum source-wise across all model architectures and sizes.
Note that the DynamixSFT-style proxy tends to achieve higher average accuracy, which aligns with their goal of optimizing the \textit{average} performance.

\subsection{DPO}
\label{sec:results-dpo}

\begin{table}[ht]
  \caption{DPO results on the Helpsteer3 dataset with a budget of 16k samples.}
  \label{tab:results-dpo}
  \centering
  \begin{tabular}{ccccccc}
    \toprule
    \multirow{2}{*}[-0.5ex]{\textbf{Model}} & \multirow{2}{*}[-0.5ex]{\textbf{Sampler}} & \multirow{2}{*}[-0.5ex]{\textbf{Signal}} & \multicolumn{2}{c}{\makecell{\textbf{Accuracy}}} & \multicolumn{2}{c}{\textbf{Margin}} \\
    \cmidrule(lr){4-5} \cmidrule(lr){6-7} & & & \textbf{\textit{Min}} & \textbf{Avg} & \textbf{Min} & \textbf{Avg} \\
    \midrule
    \multirow{4}{*}{\makecell{Qwen2.5-3B\\(Instruct)}} &
        Fixed-Weight & - & \textit{75.00} & \textit{80.57} & 0.96 & 1.28 \\ &
        Uniform & Loss & 76.76 & 80.53 & 1.08 & 1.31 \\ &
        Exp3 & Excess Loss & 77.78 & 79.98 & 1.12 & 1.40 \\ &
        AMA-R$^1$ & Excess Loss & 78.19 & 79.45 & 1.03 & 1.30 \\ &
        AMA-S$^1$ & Excess Loss & \textbf{80.25} & \textbf{82.35} & 1.07 & 1.31 \\ &
        DUCB (Ours) & Loss & \textbf{79.33} & 81.59 & 1.10 & 1.33 \\
    \midrule
    \multirow{4}{*}{\makecell{Llama3.2-3B\\(Instruct)}} &
        Fixed-Weight & - & \textit{80.24} & \textit{81.75} & 1.24 & 1.33 \\ &
        Uniform & Loss & 73.79 & 81.06 & 0.98 & 1.29 \\ &
        Exp3 & Excess Loss & 71.37 & 79.94 & 0.93 & 1.32 \\ &
        AMA-R$^1$ & Excess Loss & 74.22 & 78.00 & 0.77 & 1.15 \\ &
        AMA-S$^1$ & Excess Loss & \textbf{77.37} & 80.48 & 0.90 & 1.20 \\ &
        DUCB (Ours) & Loss & \textbf{77.42} & \textbf{81.88} & 1.04 & 1.30 \\
    \bottomrule
    \addlinespace[2pt]
    \multicolumn{7}{@{}l}{\footnotesize $^1$Faithful implementations of~\citet{corrado_automixalign_2025}.}
  \end{tabular}
\end{table}

We evaluate the DPO performance on the HelpSteer3~\cite{wang_helpsteer3_2025,wang_helpsteer3-preference_2025} dataset in~\cref{tab:results-dpo}. We compare against compare against faithful implementations of AutoMixAlign~\cite{corrado_automixalign_2025} including both AMA-R and AMA-S, and Exp3-style samplers based on excess loss.
Specifically, we train the model using uniform sampling for 100 steps, and utilize the (smoothed) training loss as a proxy to the loss of a \textit{specialist model} for each data source. Then we use the difference between the instantaneous training losses and the specialist loss proxies as the Exp3 sampling signal.
Additionally, we also compare against a \textit{fixed weight} sampler, which assigns a deterministic weight based on the specialist loss proxies for each model, respectively.

The results reveal that DUCB outperforms AMA-R on both models and remains competitive with AMA-S, which achieves higher minimum accuracy on Qwen2.5-3B.
It is also worth noting that the simple fixed-weight sampler performs exceptionally well for Llama3.2-3B, while failing to beat the uniform sampling baseline for Qwen2.5-3B.
This suggests that a relatively stationary data mixture method is not uniformly applicable for all model families without case-specific tunings. On the other hand, DUCB demonstrates its capacity to automatically detect and adapt to these varying levels of non-stationarity.

\begin{table}[ht]
    \caption{Three-seed DPO results on HelpSteer3 with Qwen2.5-3B-Instruct. We report final-checkpoint mean accuracy $\pm$ standard deviation. The paired gain compares DUCB and Uniform under matched seeds.}
    \label{tab:results-seeds}
    \centering
    \begin{tabular}{lccc}
        \toprule
        \textbf{Sampler} & \textbf{Runs} &
        \textbf{Min Acc.} & \textbf{Avg Acc.} \\
        \midrule
        Uniform & 3 & 76.76$\pm$1.65 & 80.53$\pm$0.70 \\
        DUCB (Ours) & 3 & \textbf{79.33}$\pm$\textbf{0.62} & \textbf{81.59}$\pm$\textbf{0.40} \\
        \midrule
        Paired gain & 3 & \textbf{+2.57}$\pm$\textbf{1.04} & - \\
        \bottomrule
    \end{tabular}
\end{table}

We further evaluate the performance of Qwen2.5-3B-Instruct on the Helpsteer3 dataset over three random seeds. Under the original sample budget, hyperparameters, and evaluation protocol, DUCB outperforms Uniform in minimum accuracy in all three matched Qwen2.5-3B DPO runs. As shown in~\cref{tab:results-seeds}, DUCB improves the final-checkpoint minimum accuracy by 2.57$\pm$1.04 percentage points while exhibiting lower variation across seeds. This provides controlled seed-level evidence of improved worst-source robustness without relying on aggregation across different datasets or model families.

\subsection{GRPO}
\label{sec:results-grpo}

We evaluate GRPO under a broader set of model families, scales, and sampling signals inpired by different baselines.
Specifically, \citet{ramesh_multi-task_2026} utilize Exp3 guided by training reward with stepwise reward improvement correction; \citet{wang_dump_2025} adopt switching-window UCB (SWUCB) guided by training advantages.
We focus on the Reasoning~\cite{chollet_measure_2019,katz_seemingly_2025,lin_zebralogic_2025} dataset and the CommonsenseQA~\cite{talmor_commonsenseqa_2019} dataset.

\begin{table}[ht]
  \caption{GRPO results on the Reasoning dataset with a budget of 10k samples.}
  \label{tab:results-grpo-reason}
  \centering
  \begin{tabular}{ccccc}
    \toprule
    \multirow{2}{*}[-0.5ex]{\textbf{Model}} & \multirow{2}{*}[-0.5ex]{\textbf{Sampler}} & \multirow{2}{*}[-0.5ex]{\textbf{Signal}} & \multicolumn{2}{c}{\makecell{\textbf{Accuracy}}} \\
    \cmidrule(lr){4-5} & & & \textbf{\textit{Min}} & \textbf{Avg} \\
    \midrule
    \multirow{5}{*}{\makecell{Qwen2.5-0.5B\\(Instruct)}} &
        Uniform & Reward & 14.73 & 23.01 \\ &
        Exp3 & Reward & 16.01 & 22.82 \\ &
        Exp3 & Reward + Difference~\cite{ramesh_multi-task_2026} & 16.30 & \textbf{23.23} \\ &
        SWUCB (DUMP) & Advantage~\cite{wang_dump_2025} & 13.83 & 22.61 \\ &
        DUCB (Ours) & Reward & \textbf{16.53} & 22.88 \\
    \midrule
    \multirow{5}{*}{\makecell{Qwen2.5-1.5B\\(Instruct)}} &
        Uniform & Reward & 18.89 & 28.98 \\ &
        Exp3 & Reward & 21.70 & \textbf{29.75} \\ &
        Exp3 & Reward + Difference~\cite{ramesh_multi-task_2026} & 20.91 & 28.78 \\ &
        SWUCB (DUMP) & Advantage~\cite{wang_dump_2025} & 19.45 & \textbf{29.70} \\ &
        DUCB (Ours) & Reward & \textbf{21.81} & 28.74 \\
    \midrule
    \multirow{5}{*}{\makecell{Llama3.2-1B\\(Instruct)}} &
        Uniform & Reward & 10.79 & 21.87 \\ &
        Exp3 & Reward & 11.13 & 22.70 \\ &
        Exp3 & Reward + Difference~\cite{ramesh_multi-task_2026} & 11.01 & 18.19 \\ &
        SWUCB (DUMP) & Advantage~\cite{wang_dump_2025} & 11.58 & 23.88 \\ &
        DUCB (Ours) & Reward & \textbf{15.40} & \textbf{25.07} \\
    \bottomrule
  \end{tabular}
\end{table}

\begin{table}[ht]
  \caption{GRPO results on the CommonsenseQA dataset with a budget of 6.4k samples.}
  \label{tab:results-grpo-cqa}
  \centering
  \begin{tabular}{ccccc}
    \toprule
    \multirow{2}{*}[-0.5ex]{\textbf{Model}} & \multirow{2}{*}[-0.5ex]{\textbf{Sampler}} & \multirow{2}{*}[-0.5ex]{\textbf{Signal}} & \multicolumn{2}{c}{\makecell{\textbf{Accuracy}}} \\
    \cmidrule(lr){4-5} & & & \textbf{\textit{Min}} & \textbf{Avg} \\
    \midrule
    \multirow{5}{*}{\makecell{Qwen3-8B\\(Instruct)}} &
        Uniform & Reward & \textbf{81.40} & 83.50 \\ &
        Exp3 & Reward & 77.22 & 85.52 \\ &
        Exp3 & Reward + Difference~\cite{ramesh_multi-task_2026} & 79.75 & 82.81 \\ &    
        SWUCB (DUMP) & Advantage ~\cite{wang_dump_2025} & 73.26 & 76.17 \\ &  
        DUCB (Ours) & Reward & \textbf{81.40} & \textbf{87.20} \\
    \midrule
    \multirow{5}{*}{\makecell{Qwen2.5-3B\\(Instruct)}} &
        Uniform & Reward & 65.38 & 72.27 \\ &
        Exp3 & Reward & 67.44 & \textbf{75.78} \\ &
        Exp3 & Reward + Difference~\cite{ramesh_multi-task_2026} & 61.54 & 66.41 \\ &    
        SWUCB (DUMP) & Advantage ~\cite{wang_dump_2025} & 65.38 & 71.88 \\ &  
        DUCB (Ours) & Reward & \textbf{69.77} & 75.39 \\
    \midrule
    \multirow{5}{*}{\makecell{Llama3.2-3B\\(Instruct)}} &
        Uniform & Reward & 69.23 & 73.44 \\ &
        Exp3 & Reward & 72.09 & 74.22 \\ &
        Exp3 & Reward + Difference~\cite{ramesh_multi-task_2026} & 69.23 & 72.66 \\ &    
        SWUCB (DUMP) & Advantage ~\cite{wang_dump_2025} & 69.23 & 71.88 \\ &          
        DUCB (Ours) & Reward & \textbf{74.41} & \textbf{75.78} \\
    \midrule
    \multirow{5}{*}{\makecell{GLM4-9B\\(Instruct)}} &
        Uniform & Reward & 77.91 & 84.08 \\ &
        Exp3 & Reward & 79.07 & 84.36 \\ &
        Exp3 & Reward + Difference~\cite{ramesh_multi-task_2026} & 82.56 & 85.55 \\ &    
        SWUCB (DUMP) & Advantage ~\cite{wang_dump_2025} & 81.40 & 85.55 \\ &          
        DUCB (Ours) & Reward & \textbf{84.88} & \textbf{87.74} \\
    \bottomrule
  \end{tabular}
\end{table}

For Reasoning shown in~\cref{tab:results-grpo-reason}, DUCB consistently provides the strongest robustness guarantees, achieving the highest minimum source-wise accuracies across all tested models.
For CommonsenseQA shown in~\cref{tab:results-grpo-cqa}, DUCB achieves the best or tied-best minimum source-wise accuracy across all four models.
This suggests that the robustness of DUCB is not tied to a specific architecture, tokenizer, or pretraining recipe.

The comparison with Exp3-style~\cite{ramesh_multi-task_2026} and DUMP~\cite{wang_dump_2025} baselines further highlights the importance of using a sampling signal aligned with worst-source robustness.
Specifically, samplers that are based on reward difference or advantage does not consistently improve worst-source accuracy, affirming the findings of~\cite{ramesh_multi-task_2026}.
We also observe that the gains are comparably clearer on the smaller models, indicating that they may be more vulnerable to source-wise imbalance during GRPO. On the other hand, larger models may already possess stronger prior knowledge and thus exhibit more robustness.

\subsection{Training Overhead}
\label{sec:results-overhead}

\begin{table}[ht]
    \caption{Training overhead on HelpSteer3 with Qwen2.5-3B-Instruct using the same A100 GPU, sample budget, and hyperparameters. Per-step time is reported as mean $\pm$ standard deviation. AMA runtimes cover generalist training but exclude the cost of training $K$ specialists. The overhead for~\citet{zang_adaptive_2025} is analytical rather than a wall-clock measurement.}
    \label{tab:runtime-overhead}
    \centering
    \begin{tabular}{lcl}
        \toprule
        \textbf{Sampler} & \textbf{Time/Step (s)} & \textbf{Additional Cost} \\
        \midrule
        Uniform & 21.90~$\pm$~1.29 & None \\
        AMA-R$^\dagger$ & 21.47~$\pm$~1.24 & Training $K$ source-specific specialists \\
        AMA-S$^\dagger$ & 21.72~$\pm$~1.36 & Training $K$ source-specific specialists \\
        Zang et al.$^\ddagger$ & 21~$+$~0.36~$\times\mathrm{BufSize}_t$ & $O(t)$ replay evaluations at step $t$ \\
        \textbf{DUCB (Ours)} & 21.94~$\pm$~1.14 & $O(K)$ scalar state \\
        \bottomrule
        \addlinespace[2pt]
        \multicolumn{3}{@{}l}{\footnotesize $^\dagger$ The reported runtime excludes specialist training;} \\
        \multicolumn{3}{@{}l}{\footnotesize $^\ddagger$ Complexity under a history that grows linearly with the training step.} \\
    \end{tabular}
\end{table}

We measure per-step runtime on the same A100 GPU using the three Qwen2.5-3B DPO runs. DUCB requires $21.94\pm1.14$ seconds per step, compared with $21.90\pm1.29$ for Uniform, with the difference falling within the observed variations. DUCB uses losses already computed during optimization, adds no auxiliary model evaluations, and maintains only $O(K)$ scalar state.
The generalist-training runtimes of AMA-R and AMA-S are similarly close to Uniform, but exclude their additional cost of training $K$ source-specific specialists. By contrast, the oracle of~\citet{zang_adaptive_2025} re-evaluates a growing history of collected samples under the current model. Thus, DUCB introduces little measurable online overhead while avoiding both auxiliary specialist training and replay-based model evaluations.

\subsection{Optimization Dynamics and Assumption Diagnostics}
\label{sec:results-opt}

Lipschitz regularity with respect to model parameters is commonly adopted in theoretical analyses of LLM optimization~\citep{malladi_fine-tuning_2023,bai_federated_2024,jiang_beyond_2025}. In our analysis, its primary role is to control changes in source losses between successive iterates visited by the optimization trajectory; we do not characterize a global Lipschitz constant for a transformer over its entire parameter space. We interpret~\cref{as:loss} accordingly as a trajectory-level regularity condition.

\begin{table}[ht]
    \caption{Optimization dynamics across the SFT, DPO, and GRPO experiments. For gradient norms and per-source signal variation, we report the mean and maximum of the run-level maxima. Worst-source variation is measured using exponentially smoothed training loss for SFT/DPO and reward for GRPO. Estimated worst-source (WS) switches are normalized per 1,000 steps.}
    \label{tab:results-opt}
    \centering
    \begin{tabular}{lcccc}
        \toprule
        \textbf{Task}
        & \textbf{Max LR}
        & \shortstack{\textbf{Max Grad. Norm}\\\textbf{mean / max}}
        & \shortstack{\textbf{Max Signal Variation}\\\textbf{mean / max}}
        & \shortstack{\textbf{WS Switches / 1k}\\\textbf{mean [min, max]}} \\
        \midrule
        SFT  & 2e-5 & 190.24~/~298.79 & 0.07~/~0.13 & 41~[3,124] \\
        DPO  & 1e-5 & 3.94~/~5.14     & 0.15~/~0.17 & 31~[13,53] \\
        GRPO & 1e-5 & 0.22~/~0.24     & 0.03~/~0.06 & 20~[0,32] \\
        \bottomrule
    \end{tabular}
\end{table}

In practice, small learning rates, gradient clipping, learning-rate schedules, and, where applicable, KL regularization help control parameter updates during LLM post-training. To examine whether the slowly varying loss environment assumed by our analysis is representative of our experiments, we record the learning rate and gradient norm reported by TRL or veRL, together with the maximum absolute change in each source's exponentially smoothed training signal between consecutive steps. We also report the frequency with which the empirically worst source changes during training.

As shown in~\cref{tab:results-opt}, the experiments use maximum learning rates of at most $2\times10^{-5}$, and the observed source-wise training signals change modestly between consecutive steps. At the same time, the identity of the worst source changes repeatedly across all three objectives, indicating a slowly varying but nonstationary environment. These observations are consistent with the trajectory-level regularity and slow-variation assumptions used in our analysis. They do not prove global Lipschitzness or estimate a global Lipschitz constant for the underlying neural networks.

\end{document}